\documentclass[10pt,twocolumn]{article}

\usepackage[preprint]{antintlpaper}
\usepackage{latexsym}
\usepackage{multirow}
\usepackage{fancyvrb}
\usepackage{listings}
\usepackage{colortbl}
\usepackage{float}


\newtheorem{arcdefinition}{Definition}
\theoremstyle{plain}
\newtheorem{arctheorem}{Theorem}[section]

\lstdefinelanguage{json}{
    basicstyle=\ttfamily\small,
    string=[s]{"}{"},
    stringstyle=\color{blue},
    breaklines=true,
    keywordstyle=\color{red},
    commentstyle=\itshape\color{gray!50},
}

\DefineVerbatimEnvironment{dataformat}{Verbatim}{frame=single, breaklines=true, fontsize=\small}

\newcommand{\inter}{INTER\textsuperscript{3}}

\AntPaperType{Technical Report}
\AntTitle{ARC: Fair Relative Advantage Comparison in Open-Ended Real-World Interaction}
\AntRunningTitle{ARC: Fair Relative Advantage Comparison}
\AntAuthors{Yongqi Tong\AntEqualContributor \and
  Tan Li Hui Faith\AntEqualContributor \and
  Choy Zhen Wen Marcus\AntEqualContributor \and
  Zhou Jin \and Kewei Fu \and Jiang-Ming Yang \and Jianshe Li \and Xin Zhang}
\AntAffiliations{Ant International}
\AntContact{\texttt{\{tongyongqi.yq, faith.t, marcus.choy, xiaocao.zj, fukewei.fkw,
  jmyang, zhouran.ljs, evan.zx\}@ant-intl.com}}
\AntEqualContributions
\AntDate{\today}
\AntLinks{
  \textbf{Framework Code:}\enspace
  \href{https://github.com/ant-intl/asri}{%
    \raisebox{-0.30ex}{\includegraphics[height=1.75ex]{%
      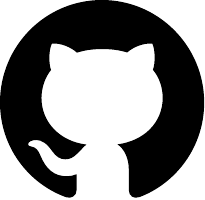}}\enspace ASRI}
  \qquad
  \textbf{Dataset:}\enspace
  \href{https://huggingface.co/datasets/ant-intl/ARC}{%
    \raisebox{-0.30ex}{\includegraphics[height=1.75ex]{%
      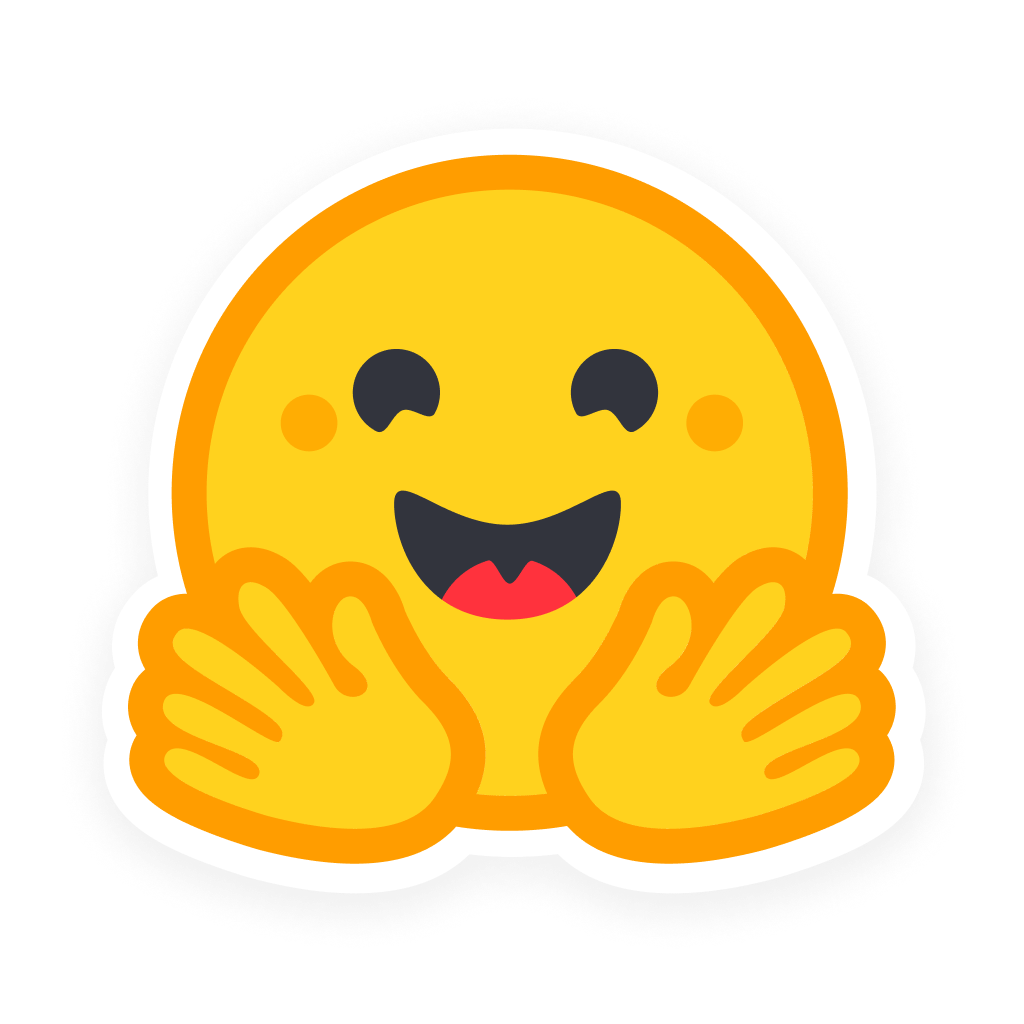}}\enspace ARC}
}
\AntKeywords{Reinforcement learning, AI agents, Tool Use, Interaction Strategies,  Advantage Estimation}
\AntAbstract{%
Open-ended real-world interaction admits multiple valid behaviors: an agent may answer directly, ask for clarification, provide progress updates, or confirm before acting. This flexibility breaks a core assumption behind group-based RL: rollouts compared within a group are no longer guaranteed to be behaviorally comparable. As a result, reward-model preferences over interaction style can distort relative advantages and steer optimization toward reward-preferred behaviors rather than context-appropriate ones. We formalize this as a \textit{reward fairness problem} and propose \textbf{ARC} (Advantage Regularization via Conditioning), a training recipe that restores fairer relative comparison through strategy-conditioned rollout grouping, together with hybrid rewards and entropy regularization. We study ARC in our proposed \inter, a novel paradigm for responsive, steerable, and execution-aware user-agent interaction that decouples user-visible communication from latent reasoning and tool use. \inter\ also provides the annotation and distillation pipeline for constructing \inter-86K, our strategy-annotated training corpus for supervised and RL training. Empirically, ARC substantially strengthens the core $\tau/\tau^2$ tool-use benchmarks, while \inter\ reduces time-to-first-token from 4.91s to 1.27s relative to a think-style baseline. Together, these results suggest that a central bottleneck in open-ended interactive learning is not only how agents are rewarded, but whether their behaviors are compared fairly in the first place. The ARC implementation and \inter-86K training data will be released.
}

\begin{document}
\makeanttitle

\section{Introduction}

\begin{figure*}
\centering
\IfFileExists{assets/v6.png}{%
  \includegraphics[width=1\linewidth]{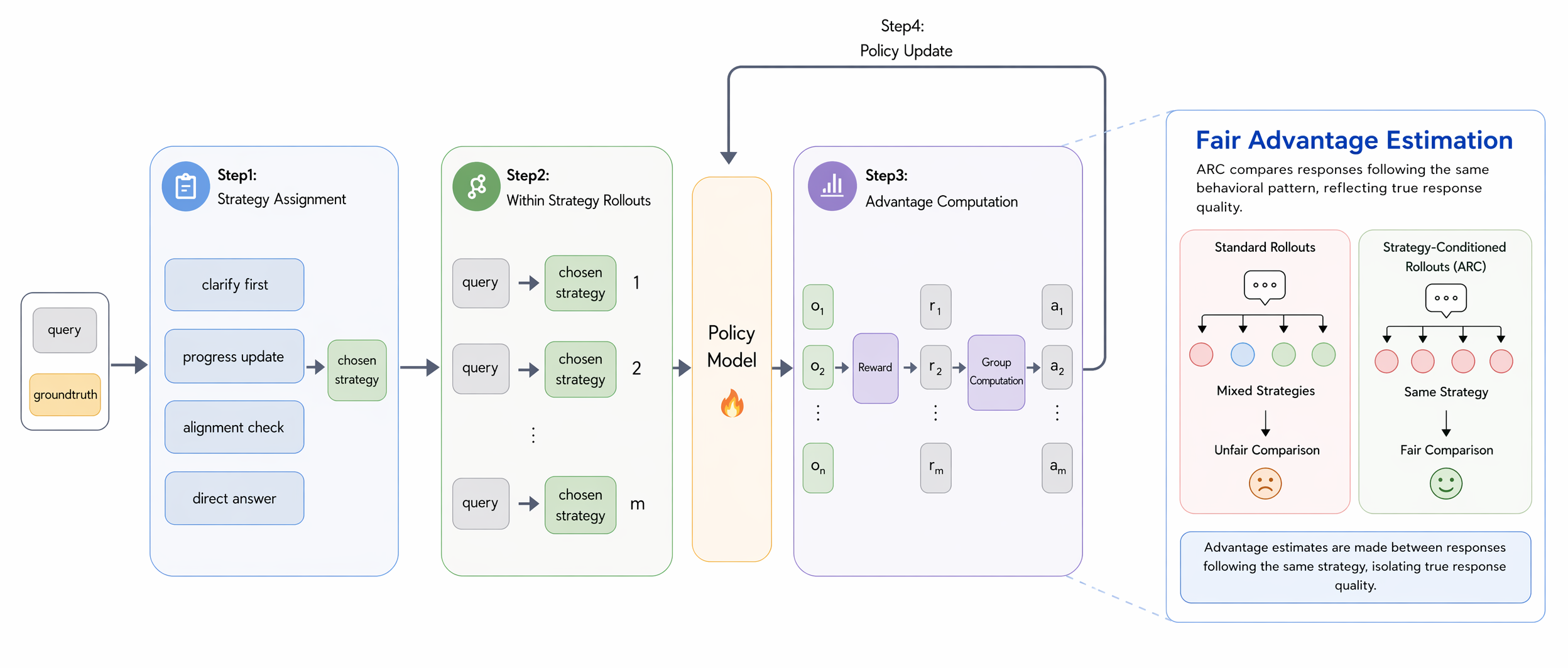}%
}{%
  \IfFileExists{assets/v5_og.png}{%
    \includegraphics[width=1\linewidth]{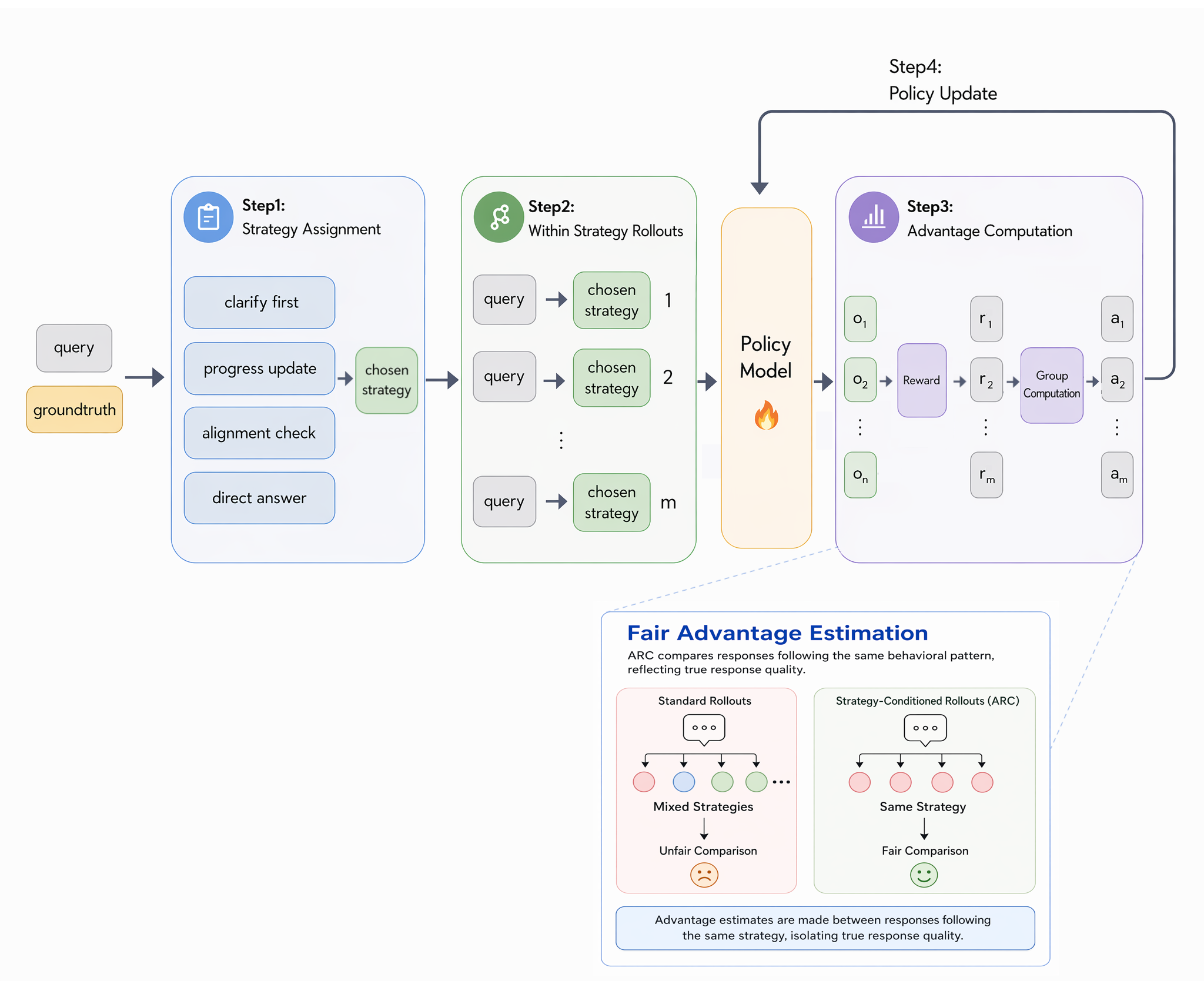}%
  }{%
    \includegraphics[width=1\linewidth]{assets/ARC_frameworkv3.png}%
  }%
}
\caption{Overview of ARC. ARC modifies group-based RL by pairing each example with a strategy instruction during training, so rollouts are compared only within strategy-conditioned groups, yielding cleaner relative advantages by eliminating cross-strategy contamination. At inference, the instruction is removed and the policy selects strategies autonomously, making ARC a general recipe for fairer advantage estimation in open-ended interaction.}
\label{fig:arc_diagram}
\end{figure*}

Group-based RL and RLHF methods learn from relative rewards within sampled groups and have become a standard post-training recipe for language models and agents~\cite{ppo2017,ahmadian2024rloo,deepseekmath2024,yu2025dapo,liu2025drgrpo}. Their signal is most interpretable when the compared rollouts occupy a locally comparable response region, so that centered rewards primarily reflect quality rather than qualitatively different behavioral modes. In practice, however, reward models are known to depend on spurious attributes such as response length and stylistic form, and RL policies can exploit these biases during optimization~\cite{odin2024,bnrm2026}. When a rollout group spans sharply different behaviors, these biases need not be uniform across the group: samples may fall in different regions of the reward model's preference landscape, so the centered reward can absorb both quality differences and region-specific reward preference. The comparison remains relative, but it is no longer fair.

This problem is especially acute in authentic real-world interaction, where many tasks are open-ended at the turn level and a single dialogue state may admit several appropriate next actions~\cite{yao2023react,tau2024,tau22025}. An agent may answer directly, ask for clarification, provide progress updates while tools run, or confirm before an irreversible action; which choice is best can depend on missing information, action reversibility, and the evolving user-agent state~\cite{tau2024,tau22025,du2024anytool,tong2026missing}. Unlike short-form reasoning domains with a single verifiable target, these interaction choices often do not have a unique standard answer even when they are all task-appropriate~\cite{tau2024,tau22025}. Comparing them inside one rollout group therefore entangles strategic diversity with reward-model preference. We formulate this confounding as the \textbf{reward fairness problem}: when multiple strategies are valid for the same prompt, cross-strategy comparison biases group-relative advantage estimation and can skew optimization toward reward-preferred interaction styles.

We address this problem with \textbf{ARC} (\textbf{A}dvantage \textbf{R}egularization via \textbf{C}onditioning). ARC assigns a training-time strategy instruction to each example, samples multiple rollouts within that strategy-conditioned comparison class, computes relative advantages only within the group, and updates the policy with our hybrid reward and entropy-regularized objective. Figure~\ref{fig:arc_diagram} gives a high-level overview. At inference, the strategy instruction is removed and the policy selects an interaction strategy autonomously. This differs from hint- or guidance-based RL methods, which use auxiliary information to reveal solution structure, privilege stronger traces, or improve exploration~\cite{sage2025,scafgrpo2024,luffy2025,exgrpo2025}. ARC uses auxiliary information for a different purpose: not to make the answer easier to find, but to compare rollouts fairly with one another.

However, existing agentic frameworks remain a weak substrate for studying this problem under authentic interaction, where users may interrupt, redirect, pause, or lose patience before execution completes, and where trajectories that achieve the same task outcome can still induce materially different interaction experiences~\cite{tau2024,tau22025,zhang2024humanai,liao2023transparency,zhang2024controllable}. Think-then-act agents typically postpone user-visible communication until the hidden reasoning-and-tool trajectory terminates, while ReAct agents intermingle reasoning, acting, and communication in a single visible trace~\cite{yao2023react,liang2025plantain,xie2025interleaved}. Both paradigms are optimized primarily for backbone capability and task success. They therefore primarily accumulate trajectories with verifiable outcomes, rather than open-ended interaction data in which multiple valid interleaving strategies may all solve the task yet induce different user experiences. We therefore build \textbf{\inter} (Interplay of Internal Reasoning, Tool Usage, and Interaction), a channel-separated framework whose system prompt explicitly defines interaction strategies aligned with the strategy families used by ARC. By separating \texttt{<answer>} spans from latent reasoning and tool execution, \inter\ turns progress updates, clarification, alignment checks, and mid-execution steering into first-class, controllable, and annotatable behaviors; this both broadens dialogue-style diversity and gives ARC cleaner rollout groups for fairer comparison. The same interface reduces TTFT from 4.91s to 1.27s relative to a think-then-act baseline. After deployment, we collect real online interaction traces in this format, augment them with curated public data, and synthesize large-scale diversified interaction trajectories aligned with ARC's strategy families, yielding \textbf{\inter-86K}, a strategy-annotated corpus of 86K examples and a realistic substrate for fair cross-strategy comparison.

Our theoretical analysis isolates \textbf{inter-strategy variance} as a source of estimator error in standard group-based RL and shows that ideal strategy conditioning removes this term from the centered-advantage variance decomposition. Our empirical results then show that, when trained on \inter-86K, ARC substantially improves open-ended agentic RL for tool-usage settings. Additional diagnostics make the mechanism more concrete: strategy hints matter during training, while omitting them at inference recovers the best overall behavior, indicating that ARC is not merely fitting to prompt-side instructions but learning a more general interaction policy. Together, \inter\ and ARC address both sides of the problem: \inter\ provides a realistic interaction framework in which diverse valid behaviors can be observed and collected, while ARC resolves how those behaviors should be compared during optimization. We view this combination as a practical step toward open-ended real-world agent training, where the central challenge is to learn from realistic interaction data while preserving fair relative comparison when multiple valid behaviors admit no single exact answer.

Our contributions are threefold:
\begin{itemize}
    \item We propose \textbf{ARC}, a strategy-conditioned group-RL recipe for mitigating unfair advantage comparison in open-ended real-world agent interaction, and we analyze how it reduces \textbf{inter-strategy variance} in group-based advantage estimation.
    \item We build \textbf{\inter}, an async streaming interaction framework that separates user-visible communication from latent reasoning and tool execution, supports explicit strategy control and user interruption, and makes open-ended interaction observable for training.
    \item We construct \textbf{\inter-86K}, a strategy-annotated training substrate from real online traces, curated public data, and synthetic trajectories, and show it enables ARC to improve open-ended agentic RL while generalizing best to hint-free inference.
\end{itemize}

\section{Related Work}

Tool-augmented agents have made rapid progress in reasoning, planning, and external action~\cite{schick2023toolformer,qin2023toolllm,yao2023react,du2024anytool,qu2024toolsurvey,tong2024weak}, while recent interleaving work has improved responsiveness by alternating between thought and partial output~\cite{liang2025plantain,xie2025interleaved}. Our departure is architectural: \inter\ separates visible interaction from latent reasoning and tool execution, so the user need not wait for the entire internal trajectory to finish before the system can communicate. This turns responsiveness from a decoding behavior into an interface property, and makes strategy diversity operational in a realistic agent loop.

On the RL side, our work builds on policy-gradient and RLHF estimators~\cite{ppo2017,ahmadian2024rloo,rafailov2023dpo,ethayarajh2024kto}, especially group-relative methods such as GRPO and its descendants~\cite{deepseekmath2024,yu2025dapo,liu2025drgrpo}. These methods are effective when samples within a group are meaningfully comparable. Our focus is a different failure mode: in open-ended agent interaction, valid rollouts may differ in strategy rather than only in quality, so cross-strategy comparison can inject reward-model preferences directly into relative advantage estimates. ARC targets this comparability problem at rollout construction time.

The closest methodological neighbors are guidance-augmented RL approaches such as SAGE, Scaf-GRPO, LUFFY, and ExGRPO~\cite{sage2025,scafgrpo2024,luffy2025,exgrpo2025}, which introduce auxiliary signals to improve exploration, mitigate sparse rewards, or reuse successful experience. ARC also uses an auxiliary instruction, but for a different reason: not to make the task easier, but to define a cleaner comparison class for group-relative optimization. Additional discussion of adjacent tool-use, interaction, and guidance literatures is deferred to Appendix~\ref{app:extended_related_work}.

\section{The \inter\ Setting}

\subsection{Interaction Interface}

ARC is evaluated in \inter, a channel-separated interaction setting for open-ended agents. The core design choice of \inter\ is to separate the \emph{interaction channel} from the \emph{execution channel}. Plain text outside \texttt{<answer>} is treated as latent reasoning, \texttt{<answer>} spans are streamed to the user, and tool calls remain structured internal actions whose results are returned to context. This makes user-visible communication available before the hidden reasoning-and-tool trajectory terminates, and it allows the same underlying task to be paired with different valid communication strategies. The mechanism is model-agnostic and requires only two implementation changes: adding \texttt{<answer>} tags to the tokenizer, and post-processing generated text to extract the user-visible spans. Figure~\ref{fig:infer3} illustrates the resulting execution pattern.

\begin{figure*}[t]
\centering
\includegraphics[width=0.9\linewidth]{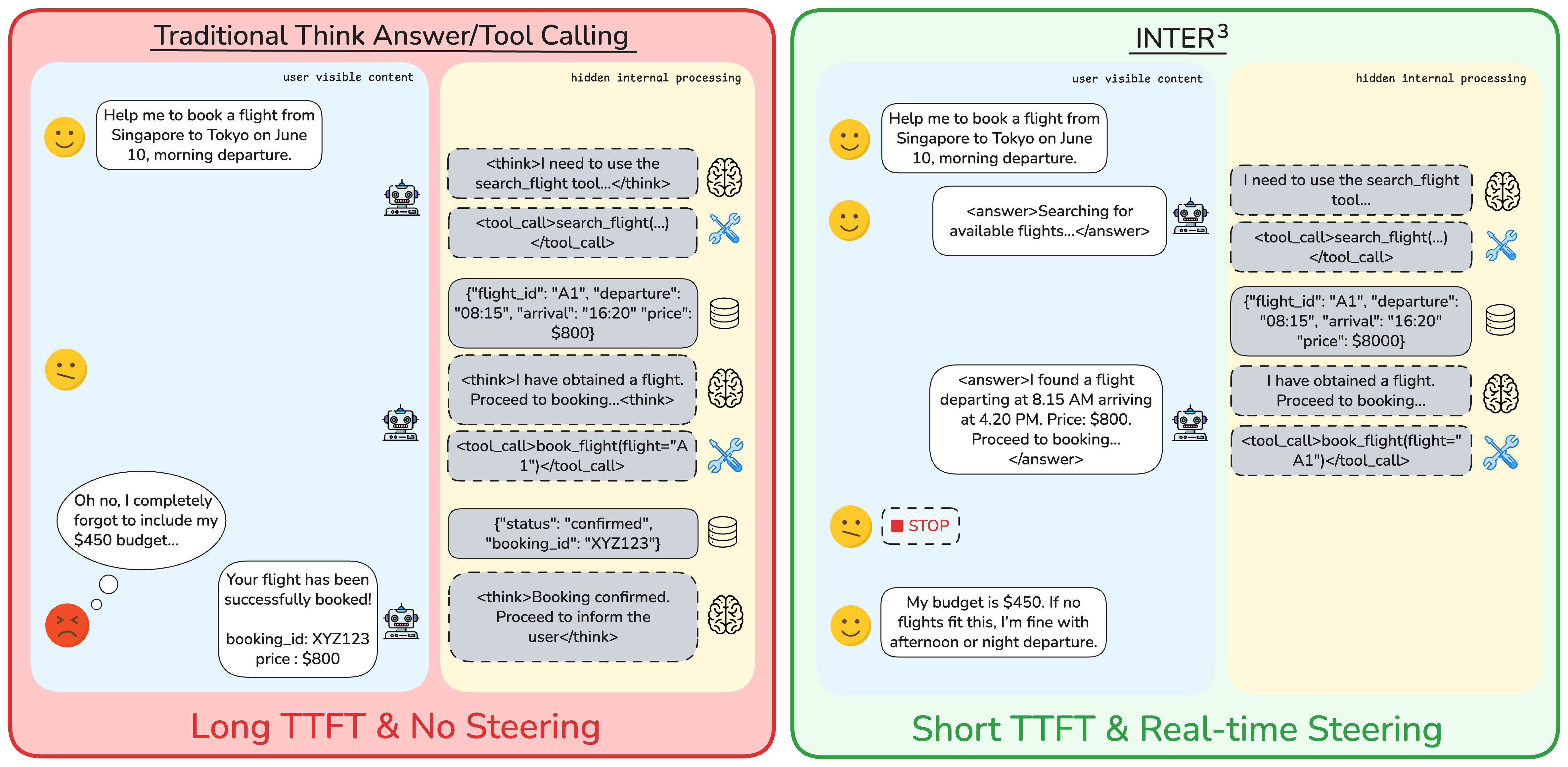}
\caption{Standard think-then-answer (left) completes reasoning and tool execution before answering, resulting in high TTFT and no steering, whereas \inter\ (right) interleaves reasoning, answering, and tool use, allowing real-time steering and faster TTFT.}
\label{fig:infer3}
\end{figure*}

\subsection{Runtime Design Rationale}


\inter\ is instantiated in a lightweight async streaming agent runtime because standard agent traces are a weak substrate for the interaction behaviors studied in this paper. Strict think-then-act pipelines delay user-visible communication until execution completes, while monolithic ReAct-style traces intermingle reasoning, acting, and communication in a single stream. Neither produces clean open-ended interaction data in which behaviors such as progress updates, clarification, alignment checks, and redirection are explicit and separately annotatable.

Because \texttt{<answer>} spans can be emitted before hidden execution completes, channel-separated interaction enables a substantially lower-latency operating regime. As shown in Figure~\ref{fig:ttft_tradeoff}, this reduces perceived latency while preserving continuous interaction during execution.

Our runtime preserves these behaviors as first-class events within a unified session. A single agent loop can continue hidden reasoning, emit user-visible \texttt{<answer>} spans, issue tool calls, and respond to user interruptions, making interaction strategy observable throughout execution rather than only in the final outcome. Two design choices are key for training. First, prompt-side strategy definitions are explicit and aligned with ARC conditioning families, so runtime control, annotation, and rollout grouping share consistent behavioral variables. Second, partial outputs, interruptions, tool results, and resumed continuations are normalized into a unified trace format, enabling post-deployment logs, curated public data, and synthetic trajectories to be converted into the same strategy-conditioned training substrate.
\begin{figure}
\centering
\includegraphics[width=0.6\linewidth]{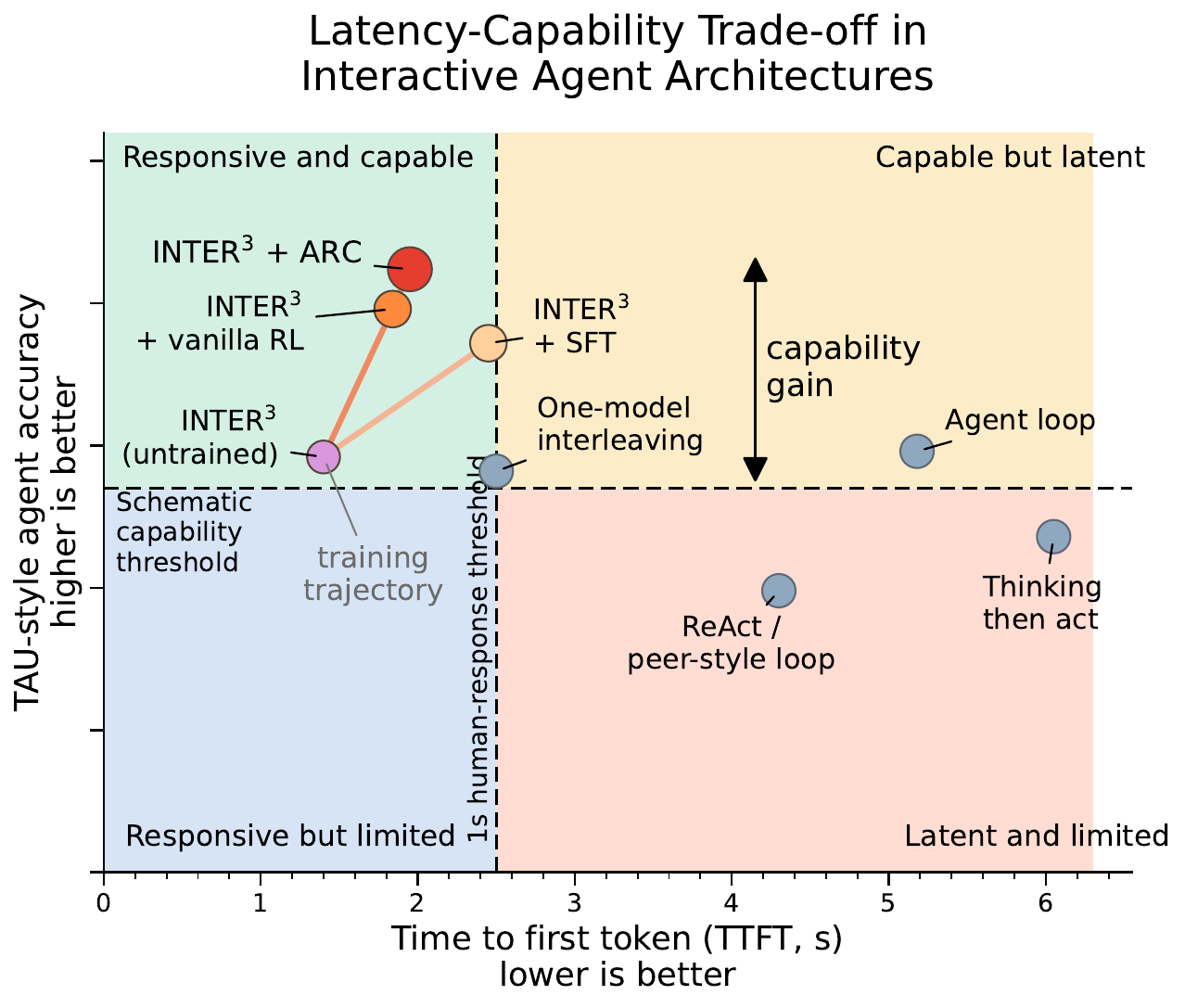}
\caption{Latency-capability trade-off for interactive agent architectures. Channel-separated interaction shifts the operating point into a substantially lower-latency regime, while ARC improves capability \emph{within} that regime rather than causing the horizontal TTFT shift itself. Non-\inter\ baselines are schematic architectural placements rather than exact benchmark claims.}
\label{fig:ttft_tradeoff}
\end{figure}


This separation makes interaction strategy explicit: the same task can be completed through different valid communication patterns while sharing the same execution substrate. That is precisely the regime in which cross-strategy comparison becomes measurable.

We organize these behaviors into four high-level strategy families: \textit{Progress Update}, \textit{Clarify First}, \textit{Alignment Check}, and \textit{Direct Answer}. ARC uses these families only as training-time comparison classes, not as inference-time requirements. We instantiate this setting with \textbf{\inter-86K}, a strategy-annotated corpus of 86.8K examples spanning tool use, multi-hop QA, and logical reasoning (57.9K SFT, 28.9K RL with strategy instructions). Since \inter\ decouples communication from computation, it reduces user-visible latency; we treat this as an architectural property rather than evidence for ARC. Full strategy definitions and dataset construction details are deferred to Section~\ref{sec:inter86k}.



\section{\inter-86K Construction}
\label{sec:inter86k}
This section presents the construction of \inter-86K, including strategy taxonomy, data sources, curation, strategy annotation, and dataset statistics.

\subsection{Strategy Taxonomy}

\inter\ organizes interaction behavior into four high-level families and nine concrete strategies as seen in Table~\ref{tab:strategies}. ARC conditions rollout groups on the high-level family during training; the finer-grained strategies are useful for analysis, and interface documentation.

\begin{table}[t]
\centering
\caption{\inter\ strategy taxonomy used for ARC conditioning, analysis, and data construction.}
\label{tab:strategies}
\tiny
\begin{tabularx}{\linewidth}{@{}llX@{}}
\toprule
\textbf{Category} & \textbf{Strategy} & \textbf{Description} \\
\midrule
\multirow[t]{5}{*}{Progress Update} & Tool Execute & Execute single or sequential tool calls; return the final answer upon completion \\
 & Parallel Tools & Execute independent tool calls simultaneously; synthesize combined results \\
 & Multi-Step Update & Stream incremental progress updates to the user between sequential sub-tasks \\
 & Silent Execution & Execute tool calls without emitting user-visible output when intermediate steps are irrelevant \\
 & Error Recovery & Handle tool failures gracefully and recover with corrective actions \\
\midrule
Clarify First & Clarify First & Ask the user for missing or ambiguous information before executing any tool call \\
\midrule
\multirow[t]{2}{*}{Alignment Check} & Alignment Check & Restate the user's intent for confirmation before executing an irreversible action \\
 & Decision Support & Present options and trade-offs for the user to make a judgment \\
\midrule
Direct Answer & Direct Answer & Answer immediately from internal knowledge without invoking any tools \\
\bottomrule
\end{tabularx}
\end{table}

\subsection{Data Sources and Curation}
\label{app:sources}

We construct \inter-86K from two sources. The first comes from real-world deployment in customer service on a large-scale global payment platform under the \inter\ runtime, where we collect online interaction traces exhibiting interruption, redirection, clarification, progress updates, and multi-step tool use. After de-identification and normalization into the \inter\ format, these traces provide realistic open-ended interaction patterns that are difficult to recover from standard offline benchmarks alone.

The second source combines curated public benchmarks with teacher-driven augmentation, distillation, and synthesis. Public tool-use and reasoning data are rewritten into the \inter\ format, expanded with strategy-conditioned variants, and supplemented with diversified synthetic interaction trajectories produced with a strong teacher model, Qwen3.5-397B-A17B~\cite{qwen3.5}. The resulting SFT split contains 57.9K examples and mixes tool-use, multi-hop QA, and logical reasoning: 34.2K tool-use examples (59.1\%) from the ToolMind collection, which aggregates public function-calling and agent benchmarks~\cite{du2024apigen,toolace2024,tau2024,xlam2024,glaive2024, when2call2025,buttoninstruct2024}; 17.2K multi-hop QA examples (29.7\%) from Musique Long Content~\cite{trivedi2022musique}; and 6.4K logical-reasoning or high-quality interleaved examples (11.1\%) from KnightsAndKnaves~\cite{xie2024memorization} and Opus Distilled~\cite{teichai2025opus}. The RL split contains 28.9K examples and is drawn entirely from tool-use data, to which we add strategy-conditioned prompts for ARC training. Figure~\ref{fig:rl_strategy_pie} summarizes the resulting domain and strategy distributions.

\subsection{SFT Data Construction}
\label{app:sft_construction}

The SFT set draws from four data families with complementary reasoning demands. The largest component ($\sim$34.2K, 59.1\%) is the ToolMind collection, which aggregates seven publicly available tool-use benchmarks spanning diverse agentic scenarios. Since these benchmarks include ground-truth chain-of-thought annotations, we use a 397B LLM (Qwen3.5-397B-A17B~\cite{qwen3.5}) to rewrite each assistant turn into the \inter\ format---interleaving internal reasoning with \texttt{<answer>} blocks and explicit tool calls---then filter malformed dialogues (invalid turn orderings, empty \texttt{<answer>} blocks). Three supplementary sources diversify the model's interleaved reasoning beyond tool calling: \textbf{KnightsAndKnaves} ($\sim$6.2K, 10.7\%) for multi-step deductive inference, \textbf{Musique Long Content} ($\sim$17.2K, 29.7\%) for multi-hop reasoning over extended passages, and \textbf{Opus Distilled} ($\sim$250, 0.4\%) for high-quality interleaved conversations generated via Claude Opus.

\subsection{RL Data Construction}
\label{app:rl_construction}

The RL set ($\sim$28.9K) reuses the seven ToolMind sub-datasets but introduces training-time strategy instructions essential for ARC training. Each conversation undergoes a two-phase process: (1)~a strategy-annotation pipeline assigns one \emph{plausible} strategy instruction from our taxonomy---\textbf{Progress Update}, \textbf{Direct Answer}, \textbf{Clarify First}, or \textbf{Alignment Check}---based on the query, context, and response; (2)~the selected instruction is injected into the system prompt, providing an explicit behavioral condition for \textit{RL+ARC} training. These prompt-side strategy definitions are aligned with the high-level strategy families used by ARC, so the interaction interface, data annotation pipeline, and rollout comparison classes share the same control variables. In contrast, \textit{baseline RL} training is performed on the same data without any injected strategy instruction.

\begin{figure}[t]
\centering
\includegraphics[width=0.85\linewidth]{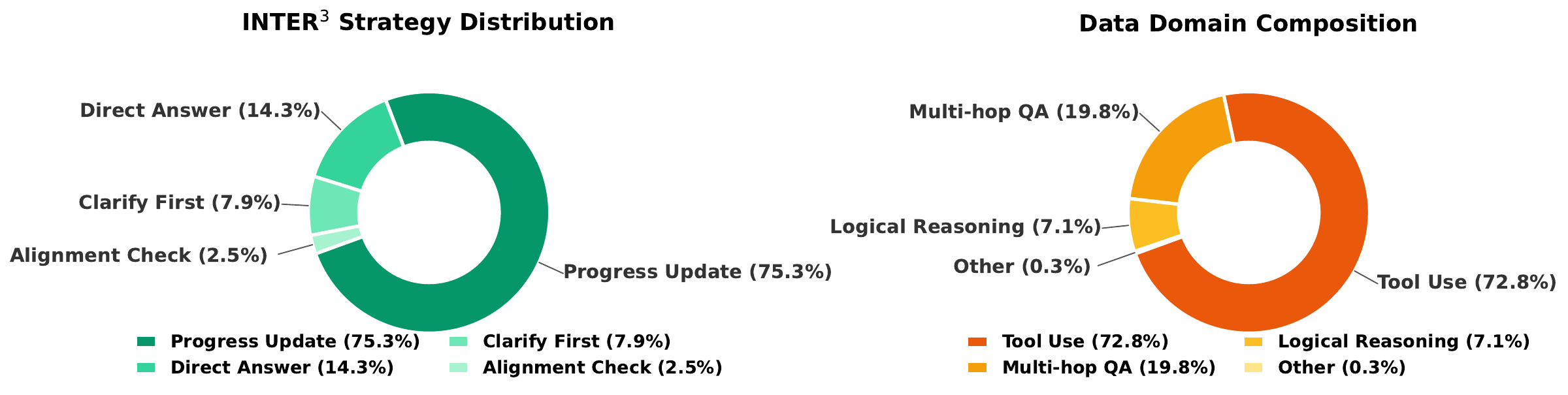}
\caption{Left: INTER\textsuperscript{3} strategy distribution in the RL training dataset. Progress Update dominates at 75.3\%, reflecting the prevalence of multi-step tool-use scenarios that benefit from incremental user updates. Right: data domain composition of the full training set (SFT + RL). Tool Use accounts for 72.8\% of examples, with Multi-hop QA (19.8\%) and Logical Reasoning (7.1\%) providing complementary reasoning diversity.}
\label{fig:rl_strategy_pie}
\end{figure}

\subsection{Strategy Annotation}
\label{app:annotation}

A core contribution of our dataset is the explicit strategy annotation for each example. Unlike prior work that treats agent behavior as monolithic, we recognize that effective interaction requires behavioral diversity adapted to context. We develop a rigorous annotation pipeline combining model-based labeling with collaborative verification.


We annotate examples on the four high-level strategies---\textbf{Progress Update}, \textbf{Clarify First}, \textbf{Alignment Check}, and \textbf{Direct Answer}.

\subsubsection{Collaborative Annotation Pipeline}

We employ a collaborative annotation approach using two large language models (Qwen3-235B-Instruct~\cite{qwen3} and Qwen3.5-27B~\cite{qwen3.5}) to ensure annotation quality and consistency. This dual-model verification captures the inherent subjectivity in strategy assignment while maintaining high inter-annotator agreement.

\textbf{Stage 1: Independent Annotation.} For each example, both models independently predict a plausible primary strategy given:
\begin{itemize}
    \item The user query
    \item The conversation history
    \item The reference assistant response
    \item The 4 high-level strategies taxonomy with detailed definitions and examples
\end{itemize}

Each model outputs: (1) a predicted strategy assignment, (2) a confidence score, and (3) a brief justification.

\textbf{Stage 2: Agreement and Disagreement Resolution.}
\begin{itemize}
    \item \textit{Agreement cases}: When both models predict the same strategy with confidence $>$0.85, we accept the annotation directly.
    \item \textit{High-confidence disagreement}: When models disagree but one has significantly higher confidence ($\Delta > 0.15$), we accept the higher-confidence prediction.
    \item \textit{Low-confidence or ambiguous disagreement}: We escalate to human review. Three trained annotators independently label the example. We determine the final annotation using a majority vote.
\end{itemize}

\subsection{Dataset Statistics}
\label{app:statistics}

Table~\ref{tab:final_stats} presents comprehensive statistics of the final dataset.

\begin{table}[ht]
\centering
\small
\caption{\textsc{Inter-86K} summary statistics. Token counts use whitespace tokenization.}
\label{tab:final_stats}
\begin{tabular}{lr}
\toprule
\textbf{Statistic} & \textbf{Value} \\
\midrule
Total examples & 86,796 \\
\midrule
\multicolumn{2}{l}{\textit{Token Statistics}} \\
Average input tokens & 702.1 \\
Average output tokens & 264.5 \\
Average total tokens & 966.6 \\
\midrule
\multicolumn{2}{l}{\textit{Tool Statistics}} \\
Examples with tool calls & 54,856~(63.2\%) \\
Average tools per example & 4.9 \\
Max tools in single example & 38 \\
\midrule
\multicolumn{2}{l}{\textit{Conversation Statistics}} \\
Single-turn conversations & 55,376~(63.8\%) \\
Multi-turn conversations & 31,420~(36.2\%) \\
Average turns (multi-turn) & 12.5 \\
\midrule
\multicolumn{2}{l}{\textit{\texttt{<answer>} Tag Statistics}} \\
Examples with \texttt{<answer>} tags & 86,780~(100.0\%) \\
Average \texttt{<answer>} segments & 3.48 \\
Average tokens in \texttt{<answer>} & 25.3 \\
\bottomrule
\end{tabular}
\end{table}

\section{Our Methodology: Advantage Regularization via Conditioning}
\label{sec:arc}

While the \inter\ setting makes interaction strategy observable, it does not solve the core RL challenge: how to compare responses fairly when they follow different communication patterns. We present \textbf{ARC} (\textbf{A}dvantage \textbf{R}egularization via \textbf{C}onditioning), a conditioning-based RL method for open-ended agent training. ARC's core mechanism is \emph{strategy-conditioned rollout grouping}, which changes the comparison unit in group-based RL: relative advantages are computed within a strategy-conditioned subspace rather than across heterogeneous behaviors.


\subsection{Problem: Unfair Advantages in Multi-Strategy RL}

In group-based RL, given prompt $x$, we sample $N$ responses $\{y_1, \ldots, y_N\}$ from $\pi_\theta(\cdot | x)$ and compute advantages $\hat{A}_i = r_i - \bar{r}$, assuming comparable samples. When responses follow different strategies, this breaks: reward model bias contaminates the advantage signal, skewing policy updates toward reward-preferred behaviors regardless of task appropriateness.

\begin{arcdefinition}[$\delta$-Reward Fairness]
A reward model $\mathcal{R}$ is \textit{$\delta$-fair} w.r.t.\ strategy set $\mathcal{S}$ if for any prompt $x$ and responses $y_i, y_j$ following different strategies with equal quality: $|\mathcal{R}(x, y_i) - \mathcal{R}(x, y_j)| \leq \delta$. A reward model is \textit{fair} if $\delta = 0$, and \textit{unfair} when $\delta > 0$ introduces systematic bias across strategies.
\end{arcdefinition}

In practice, reward models violate this fairness property due to exposure, length, and style bias---especially pronounced in agent settings where interaction appropriateness lacks clear ground truth.

\subsection{Strategy-Conditioned Advantage Estimation}

The following results analyze a stylized prompt-conditional reward decomposition of the form $r=\mu_s(x)+\epsilon$. They isolate one source of estimator variance in group-relative advantages and should not be read as end-to-end convergence guarantees for GRPO. This estimator-centric viewpoint is closer to classical variance-reduction analyses for policy gradients ~\cite{greensmith2004variance} and recent RLHF-side discussions of estimator behavior~\cite{ahmadian2024rloo,liu2025drgrpo} than to a full convergence theory. For clarity, the theoretical comparison uses equal group size across the conditioned and unconditioned cases.

\begin{arctheorem}[Idealized Variance Amplification]
\label{thm:variance_amplification}
Under standard group-based RL sampling, the advantage variance is $\mathrm{Var}[\hat{A}_i] = (\sigma^2_{\text{inter}} + \sigma^2_{\text{intra}})(1 - 1/N)$, where $\sigma^2_{\text{intra}}$ captures within-strategy noise and $\sigma^2_{\text{inter}}$ captures between-strategy variance from reward model bias.
\end{arctheorem}

The inter-strategy component $\sigma^2_{\text{inter}}$ \textit{persists} under standard sampling: no amount of oversampling can eliminate it. Our solution is to condition each rollout group on a specific strategy.

\begin{arcdefinition}[Strategy-Conditioned Sampling]
For each training example, we construct a strategy-conditioned prompt $p^*$ by attaching a strategy instruction $s^*$ from our interaction taxonomy, then sample $M$ responses from $\pi_\theta(\cdot | x, s^*)$.
\end{arcdefinition}

\begin{arctheorem}[Idealized Variance Reduction via Conditioning]
\label{thm:variance_reduction}
Under strategy-conditioned sampling with the same group size: $\mathrm{Var}[\hat{A}_i | s^*] = \sigma^2_{\text{intra}} \cdot (1 - 1/N)$. Under perfect conditioning, the between-strategy mean-shift term is absent from this centered-advantage variance decomposition.
\end{arctheorem}

\textbf{Intuition.} Heuristically, strategy conditioning emphasizes within-strategy reward dependence $I(r; y \mid x, s)$ while suppressing the strategy-linked term $I(r; s \mid x)$ that can contaminate cross-strategy comparison. Unlike hint-based RL methods~\cite{sage2025,scafgrpo2024} that use hints to guide toward correct answers, we use strategy instructions to enforce cleaner within-strategy comparisons across diverse behaviors. See Appendix~\ref{app:proofs} for proofs and analysis.

Training-time strategy instructions do introduce a train-inference mismatch because they are removed at deployment. We study progressive instruction removal empirically in Section~\ref{sec:curriculum_analysis} as a further analysis rather than as part of ARC itself.

\subsection{Entropy Regularization for Multi-Channel Generation}
\label{sec:training_objective}

The ARC training objective combines strategy-conditioned policy gradients with an entropy bonus:
\begin{equation}
\mathcal{L}(\theta) = -\sum_{i=1}^{M} \hat{A}_i^{(s^*)} \cdot \log \pi_\theta(y_i^{(s^*)} | x, s^*) - \beta \cdot H\!\left(\pi_\theta(\cdot | x, s^*)\right).
\end{equation}
\textbf{Why entropy regularization is critical for \inter.} The interleaved output format requires balancing three channels---internal reasoning, user-facing \texttt{<answer>} tags, and tool calls. Without the entropy bonus ($\beta = 0$), we observe \textit{entropy collapse}~\cite{jin2025revisiting}: the policy converges to emitting redundant \texttt{<answer>} blocks that receive marginal format rewards but carry no meaningful content. The entropy bonus counteracts this collapse by maintaining stochasticity across all output channels, encouraging diverse exploration of valid interleaving patterns.

\subsection{Idealized Sample-Efficiency View}

\begin{arctheorem}[Idealized Gradient-Sample Scaling]
\label{thm:sample_complexity}
Under the same stylized assumptions, if gradient-estimation error scales with the second moment of the score-function estimator, then the rollout requirement to estimate $\nabla J(\theta)$ up to accuracy $\varepsilon$ scales as $n_{\text{std}} = O\!\left(\frac{(\sigma^2_{\text{intra}} + \sigma^2_{\text{inter}})\log(1/\delta)}{\varepsilon^2}\right)$ for standard GRPO and $n_{\text{ARC}} = O\!\left(\frac{\sigma^2_{\text{intra}}\log(1/\delta)}{\varepsilon^2}\right)$ for ARC.
\end{arctheorem}

\textbf{Corollary 1.} Under the same assumptions, the implied sample-efficiency ratio is $n_{\text{std}} / n_{\text{ARC}} = 1 + \sigma^2_{\text{inter}} / \sigma^2_{\text{intra}}$. This is best read as an idealized scaling comparison rather than a full convergence guarantee, in the same spirit as estimator-level policy-gradient analyses~\cite{greensmith2004variance,yuan2022finite}. We use it as an estimator-level interpretation of why cleaner within-strategy comparison can reduce the gradient-sampling burden, rather than as a calibrated empirical speedup claim. See Appendix~\ref{app:proofs} for the derivation and assumptions.

\subsection{Implementation}

ARC consists of four steps (Figure~\ref{fig:arc_diagram}):

\begin{enumerate}
    \item \textbf{Strategy Instruction Assignment}: each training example is paired with one strategy instruction $s^*$ from the interaction taxonomy.
    \item \textbf{Within-Strategy Sampling}: sample $M$ rollouts from $\pi_\theta(\cdot \mid x, s^*)$.
    \item \textbf{Advantage Computation}: compute advantages within each strategy group.
    \item \textbf{Policy Update}: update the policy using the entropy-regularized objective.
\end{enumerate}

At inference, no strategy instruction is provided; the model autonomously selects appropriate strategies. Section~\ref{sec:curriculum_analysis} further analyzes instruction-removal curricula as an auxiliary experiment.

\section{Experiments}

\subsection{Experimental Setup}

\textbf{Model and training.} We use Qwen3-8B~\cite{qwen3} trained with GRPO~\cite{deepseekmath2024} in no-think mode, warm-started from the same \inter\ SFT checkpoint for all RL backbones. ARC's core mechanism is strategy-conditioned rollout grouping; in the full \inter\ instantiation additionally apply entropy regularization as a stabilizer for multi-channel generation.
Unless explicitly varied, all methods use the same final reward construction, selected based on ablations in Appendix~\ref{app:reward_ablation}. Additional ablations on entropy regularization are provided in Appendix~\ref{app:arc_ablation}.



\textbf{Benchmarks.} We evaluate on two dimensions: (1)~\textit{In-domain agentic capabilities} using tau-bench~\cite{tau2024} and tau2-bench~\cite{tau22025} for multi-turn tool calling in airline, retail, and telecom scenarios; (2)~\textit{Out-of-domain reasoning} using Arena-Hard~\cite{arenahard}, AIME 2026~\cite{matharena2025}, GPQA-Diamond~\cite{gpqa2024}, IFBench~\cite{ifbench2025}, and HMMT~\cite{matharena2025}.

\textbf{Baseline methods.} We benchmark ARC against two classes of baselines: (1)~\textit{Qwen3-8B-noThink / Think}, minimal reasoning variants to quantify ARC’s gains; (2)~\textit{Standard RL backbones} (PPO~\cite{ppo2017}, DAPO~\cite{yu2025dapo}, GRPO~\cite{deepseekmath2024}), representing widely-used reinforcement learning methods.

\subsection{Main Results}
\label{sec:main_results}

\begin{figure}[t]
\centering
\includegraphics[width=1\linewidth]{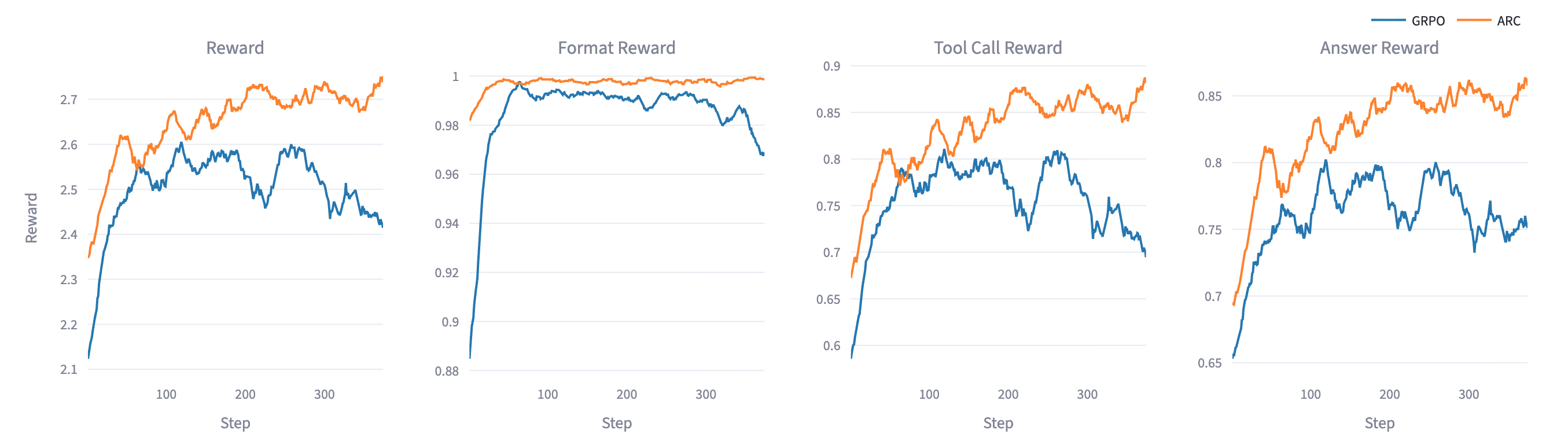}
\caption{Training reward curves for GRPO and ARC.}
\label{fig:reward_curves}
\end{figure}

\begin{table*}[t]
\centering
\caption{Main results. ARC delivers clearest gains on in-domain tool use, especially on GRPO. Avg. is the mean over displayed metrics; red parentheses show Avg. change vs. backbone. TTFT ($\downarrow$) averages available $\tau$-bench latency measurements.}
\label{tab:main_results}
\scalebox{0.65}{
\begin{tabular}{lcccccccccccc}
\toprule
\multirow{2}{*}{\textbf{Method}} & \multirow{2}{*}{\textbf{Avg.}} & \multirow{2}{*}{\textbf{TTFT $\downarrow$}} &
  \multicolumn{5}{c}{\textbf{Tool Calling}} &
  \multicolumn{3}{c}{\textbf{Reasoning}} &
  \textbf{\begin{tabular}[c]{@{}c@{}}Instruction\\Following\end{tabular}} &
  \textbf{Alignment} \\
\cmidrule(lr){4-8}\cmidrule(lr){9-11}\cmidrule(lr){12-12}\cmidrule(lr){13-13}
 & & & \textbf{$\tau$-airline} & \textbf{$\tau$-retail} & \textbf{$\tau^2$-airline} & \textbf{$\tau^2$-retail} & \textbf{$\tau^2$-telecom} &
  \textbf{\begin{tabular}[c]{@{}c@{}}AIME\\2026\end{tabular}} &
  \textbf{\begin{tabular}[c]{@{}c@{}}GPQA-D\end{tabular}} &
  \textbf{\begin{tabular}[c]{@{}c@{}}HMMT\\2025\end{tabular}} &
  \textbf{IFBench} & \textbf{ArenaHard} \\
\midrule
\multicolumn{13}{l}{\textit{Baselines}} \\
\quad Qwen3-8B-noThink & 22.85 & 0.05s & 12.00 & 29.86 & 14.61 & 36.55 & 17.80 & 17.08 & 45.71 & 12.92 & 24.83 & 17.14 \\
\quad Qwen3-8B-Think   & 32.82 & 4.91s & 28.00 & 36.81  & 29.75 & 38.71 & 23.46 & 47.92 & 53.54 & 20.83 & 19.43 & 29.83 \\
\midrule
\multicolumn{13}{l}{\textit{RL Backbones}} \\
\quad PPO              & 27.49 & 0.45s & 35.33 & 41.45 & 33.33 & 38.89 & 19.01 & 30.83 & 38.70 & 8.75 & 19.73 & 8.88 \\
\quad PPO + ARC (Ours) & 28.57 {\scriptsize\textcolor{red}{(+1.08)}} & 0.78s & 39.33 & 46.09 & 41.61 & 44.44 & 20.76 & 19.17 & 41.10 & 10.00 & 14.29 & 8.92 \\
\quad DAPO             & 28.61 & 0.62s & 31.33 & 42.32 & 35.33 & 42.11 & 19.37 & 36.67 & 39.02 & 13.33 & 12.24 & 14.36 \\
\quad DAPO + ARC (Ours)& 29.92 {\scriptsize\textcolor{red}{(+1.31)}} & 0.82s & 34.67 & 41.74 & 34.50 & 40.64 & 19.64 & 41.25 & 43.94 & 12.50 & 17.35 & 12.92 \\
\quad GRPO             & 28.09 & 0.61s & 31.33 & 40.29 & 36.67 & 40.64 & 17.84 & 31.67 & 40.91 & 9.17 & 18.71 & 13.66 \\
\quad GRPO + ARC (Ours)& 33.46 {\scriptsize\textcolor{red}{(+5.37)}} & 1.27s & 44.00 & 50.00 & 48.00 & 45.61 & 21.05 & 40.83 & 41.41 & 12.50 & 15.99 & 15.18 \\
\bottomrule
\end{tabular}
}
\end{table*}

\textbf{In-domain performance.} Table~\ref{tab:main_results} shows that ARC consistently improves performance across all RL backbones, with the largest gains for GRPO (+5.37 average points). Improvements are especially pronounced in in-domain tool-calling tasks: for GRPO, $\tau$-airline rises from 31.33 to 44.00, $\tau$-retail from 40.29 to 50.00, and $\tau^2$-airline from 36.67 to 48.00, with smaller gains in $\tau^2$-retail and $\tau^2$-telecom. These results demonstrate ARC’s ability to enhance multi-turn tool use and in-domain capabilities, particularly with stronger RL backbones.

\textbf{Out-of-domain reasoning and tradeoffs.} On out-of-domain reasoning, instruction following, and alignment benchmarks, ARC shows more nuanced effects. GRPO + ARC improves reasoning on AIME 2026 from 31.67 to 40.83 and provides modest gains in instruction-following and alignment metrics, while improvements for PPO and DAPO backbones are mixed, suggesting that ARC's benefits depend on the underlying optimization dynamics. These results show that while ARC's primary impact is on structured tool use, it can also unlock meaningful out-of-domain reasoning gains under suitable training regimes.

\textbf{Training dynamics and robustness.} Figure~\ref{fig:reward_curves} shows that ARC's benefits extend beyond final performance to training dynamics. While GRPO peaks mid-training and then declines—especially in tool-call and answer rewards—ARC maintains or improves across the same budget. Since format reward saturates for both methods, this divergence reflects execution and answer quality rather than syntactic compliance. This suggests that standard group-based RL initially learns useful behaviors but then exploits reward model biases, whereas ARC's strategy-conditioned comparison prevents degradation by preserving fair advantage signals throughout training.

\subsection{Small-Model Case Study}


\begin{table}[t]
\centering
\caption{Qwen3-4B case study. ARC remains effective at smaller scale and outperforms no-think, think, and 4B GRPO baselines on the in-domain tool-use suite.}
\label{tab:small_model_case}
\resizebox{\linewidth}{!}{
\begin{tabular}{lcccccc}
\toprule
\textbf{Method} & \textbf{Avg.} & \textbf{$\tau$ airline} & \textbf{$\tau$ retail} & \textbf{$\tau^2$ airline} & \textbf{$\tau^2$ retail} & \textbf{$\tau^2$ telecom}  \\
\midrule
Qwen3-4B-noThink  & 22.38 & 20.00 & 22.32 & 21.40 & 26.31 & 21.90 \\
Qwen3-4B-Think  & 29.54 & 27.51 & 37.10 & 28.07 & 29.21 & 25.83 \\
GRPO (4B)  & 22.33 & 29.33 & 10.72 & 36.94 & 15.96 & 18.69 \\
GRPO+ARC (4B) & 34.23 & 29.33 & 40.00 & 32.00 & 39.77 & 30.03  \\
\bottomrule
\end{tabular}
}
\end{table}

Table~\ref{tab:small_model_case} shows that the agentic benefit of ARC is not specific to the 8B scale. Even at 4B, ARC remains stronger than no-think, think, and the corresponding 4B GRPO baseline on the tool-use suite, increasing the five-task average by about 53\% over 4B no-think. This finding suggests that ARC improves performance by strengthening the training comparison signal in agentic settings, rather than simply leveraging larger model capacity.

\subsection{Curriculum Learning over Training-Time Strategy Instructions}
\label{sec:curriculum_analysis}

ARC uses strategy instructions during training but removes them at inference, motivating a curriculum-learning question: should the conditioning signal be gradually weakened so the policy learns more autonomous strategy selection without losing ARC's variance-control benefit. Inspired by curriculum-learning methods in RL~\cite{bengio2009curriculum,florensa2017reverse,florensa2018goal,matiisen2017tscl}, we study this as a diagnostic analysis of the trade-off between train-inference alignment and within-strategy comparability.

We compare three settings: (1) \textit{No removal}, the default ARC setup that always retains the strategy instruction; (2) \textit{Linear removal}, where the instruction-drop probability increases over training according to Appendix~\ref{app:proofs}, $p_d(t)=p_{\min}+(p_{\max}-p_{\min})\frac{t}{T}$; and (3) \textit{Constant removal}, which drops instructions with fixed probability $p_d=0.20$ throughout training.

\begin{figure}[t]
\centering
\includegraphics[width=0.9\linewidth]{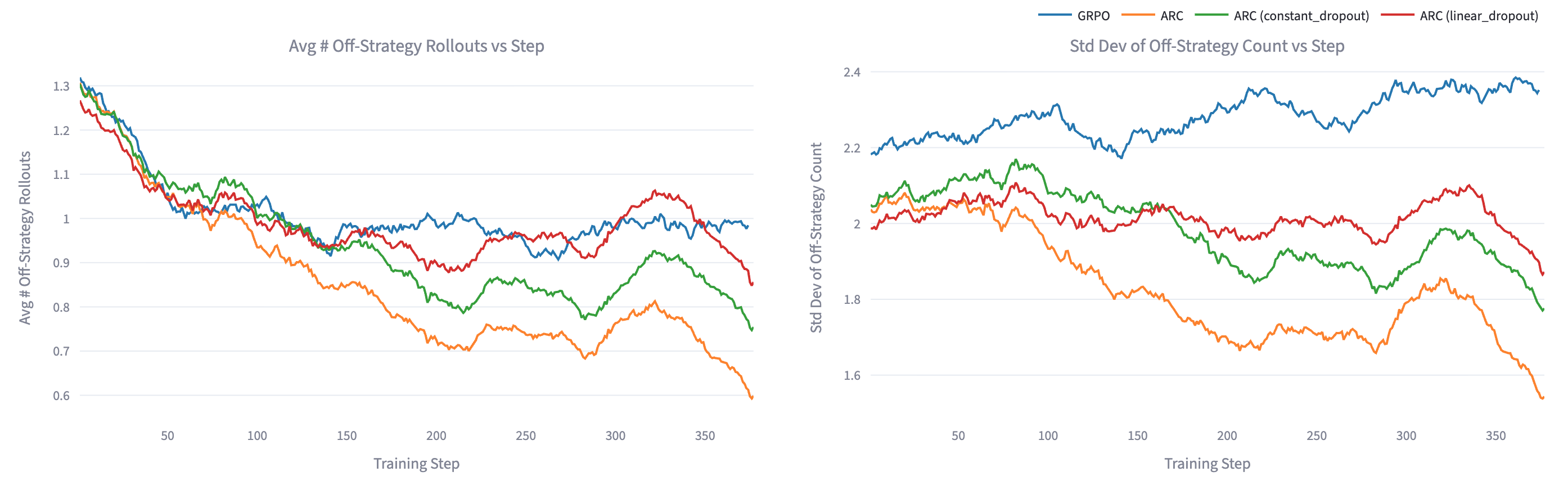}
\caption{Training-time instruction removal underperforms the default ARC setup. Left: average number of off-strategy rollouts per group, where an rollout has a predicted strategy different from the injected strategy. Right: standard deviation of off-strategy counts across groups. Lower values are better in both panels, indicating more consistent within-strategy rollout grouping.}
\label{fig:curriculum_ablation}
\end{figure}

\begin{table}[t]
\centering
\caption{Curriculum learning over strategy-instruction removal does not improve ARC: no removal performs best, linear removal is intermediate, and constant removal is weakest.}
\label{tab:curriculum_summary}
\resizebox{\linewidth}{!}{
\begin{tabular}{lcccccc}
\toprule
\textbf{Training Setting} & \textbf{Avg.} & \textbf{$\tau$ Avg.} & \textbf{$\tau^2$ Avg.} & \textbf{Reasoning Avg.} & \textbf{IFBench} & \textbf{ArenaHard} \\
\midrule
No removal (ARC default) & 29.59 & 47.00 & 38.22 & 31.58 & 15.99 & 15.18 \\
Linear removal & 27.28 & 38.21 & 31.83 & 34.12 & 16.33 & 15.92 \\
Constant removal & 26.97 & 34.97 & 32.99 & 33.21 & 18.71 & 15.01 \\
\bottomrule
\end{tabular}
}
\end{table}


Table~\ref{tab:curriculum_summary} and Figure~\ref{fig:curriculum_ablation} show that progressively removing strategy instructions does not improve ARC in our setting. Default ARC achieves the strongest $\tau/\tau^2$ tool-use performance and the lowest off-strategy mean and variance, indicating more stable within-strategy rollout grouping. Linear removal provides a partial trade-off by modestly improving some reasoning metrics at the cost of agentic performance, while constant removal performs worst overall. These results suggest that strategy instructions are most effective as a persistent training-time variance-control mechanism rather than a signal to be annealed away. This is consistent with ARC's mechanism: strategy instructions are most useful as a training-time variance-control device, and annealing them away too early weakens within-strategy comparability instead of improving the final policy.

\subsection{Strategy Scalability Analysis}

We conduct an additional ablation study of strategy scaling in Table~\ref{tab:strategy_scaling}, which reveals a clear performance trajectory as strategy components are incrementally integrated, moving from the highest to the lowest training data volume. The \textit{Avg.} shows a consistent upward trend, achieving an impressive 36.1\% total improvement. The most significant gains are concentrated in agentic benchmarks, where performance on $\tau$-bench and $\tau^2$-bench surged by 99\% and 71\%, respectively, as the model transitioned from a single strategy to the full suite.
\begin{figure}[t]
\centering
\includegraphics[width=0.8\linewidth]{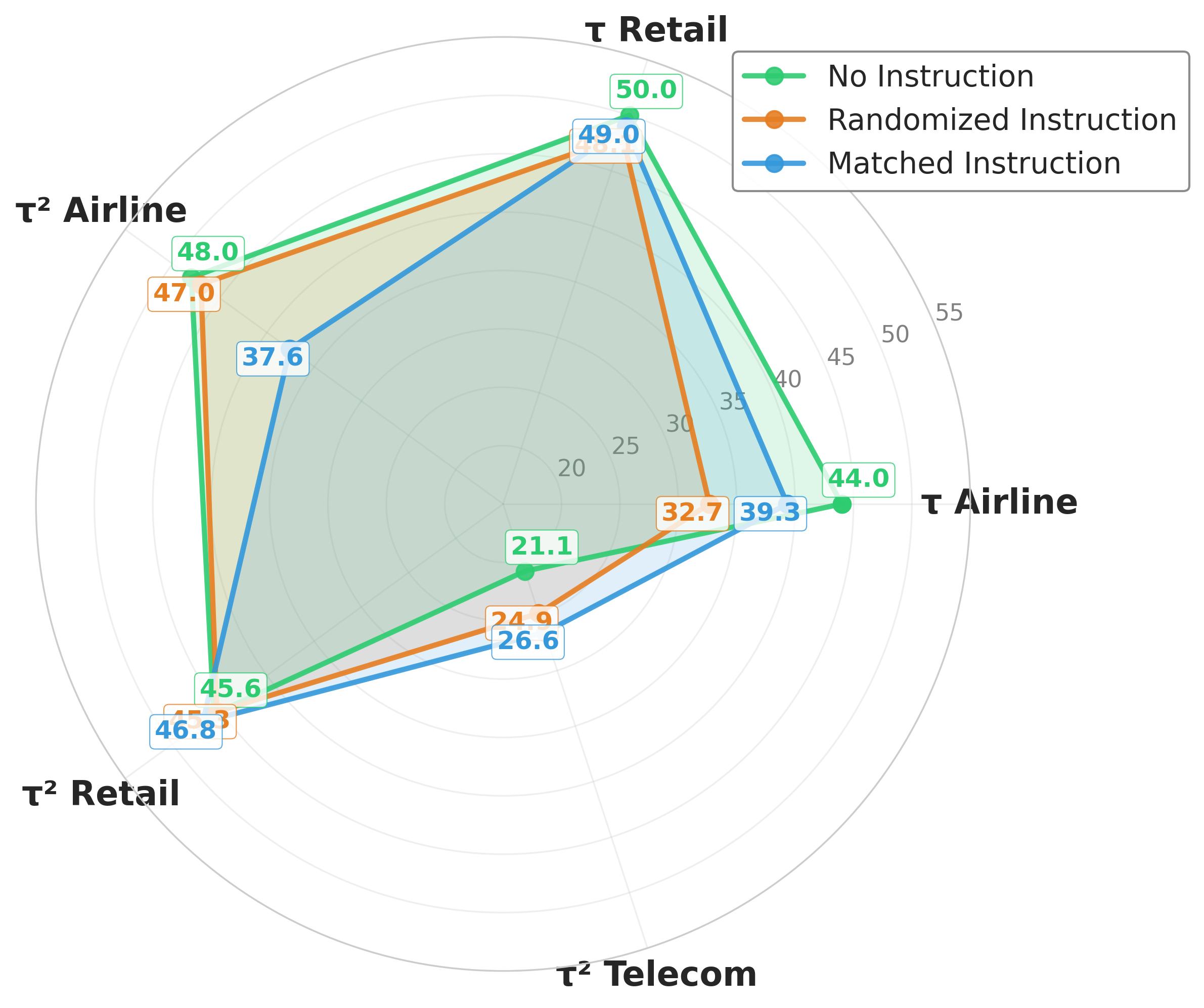}
\caption{Domain-instruction interaction analysis. Inference-time strategy hints show no uniformly superior deployment interface.}
\label{fig:inference_instruction_ablation}
\end{figure}

These results highlight not only the effectiveness of the method in leveraging heterogeneous strategies, but also its scalability: as additional strategies are added, the model continues to improve, demonstrating that even minority strategies—such as \textit{Alignment Check}, which constitutes only 2.5\% of the data—contribute meaningfully to overall capabilities. This highlights the potential for accommodating larger strategy sets without performance degradation.

\begin{table}[t]
\centering
\caption{Strategy scaling ablation. Incremental strategy scaling consistently drives overall performance gains: the full strategy suite performs best, while reasoning tasks peak with three strategies and agentic benchmarks show the most dramatic sensitivity to strategy diversity.}
\label{tab:strategy_scaling}
\resizebox{\linewidth}{!}{
\begin{tabular}{lcccccc}
\toprule
\textbf{Strategy Segment} & \textbf{Avg.} & $\tau$ \textbf{Avg.} & $\tau^2$ \textbf{Avg.} & \textbf{Reasoning Avg.} & \textbf{IFBench} & \textbf{ArenaHard} \\
\midrule
Progress Update & 21.74 & 23.62 & 22.36 & 34.01 & 16.33 & 14.27 \\
Progress Update + Direct Answer & 24.91 & 31.67 & 32.46 & 27.79 & 17.69 & 14.96 \\
Progress Update + Direct Answer + Clarify First & 25.60 & 31.04 & 28.97 & 35.21 & 16.67 & 16.13 \\
Full Suite (4) & 29.59 & 47.00 & 38.22 & 31.58 & 15.99 & 15.18 \\
\bottomrule
\end{tabular}
}
\end{table}

\subsection{Additional Diagnostics on Strategy Instructions}

We evaluate whether the final ARC policy depends on explicit strategy hints at deployment, or whether the strategy behavior has already been internalized during training. We compare matched, randomized, and removed instructions on the same ARC checkpoint across all $\tau$-bench and $\tau^2$-bench submetrics to isolate the effect of prompt-side cues.



Conclusively, Figure~\ref{fig:inference_instruction_ablation} shows that inference-time strategy hints do not improve the deployed policy. The baseline with no strategy hint removal managed to achieve the highest overall average (41.73), outperforming both Randomized Instruction (39.58) and Matched Instruction (39.85). Domain-level analysis reveals that matched instructions help on $\tau$ Airline but hurt on $\tau^2$ Airline, while all settings struggle on $\tau^2$ Telecom. This domain-dependent variability confirms that inference-time hints lack a reliable deployment benefit. Strategy instructions serve their primary function during training by structuring comparison groups, not as persistent cues that should remain at inference.

\section{Conclusion}

Open-ended agent interaction often admits multiple valid behaviors, yet standard group-based RL compares rollouts within a shared relative-reward pool. We argue this induces a reward fairness issue: advantage estimates become confounded when heterogeneous interaction strategies are normalized together, rather than reflecting intrinsic quality differences.

ARC addresses this failure by conditioning the comparison class during rollout construction, ensuring that each group contains a single strategy family. This targeted modification improves the interpretability of relative advantages in multi-strategy regimes, rather than altering the underlying policy optimization machinery. Complementarily, the \inter\ setting makes this issue observable in practice by separating user-facing communication from latent reasoning and tool execution, thereby exposing diverse valid interaction strategies within the same task family.

Empirically, ARC is most effective in in-domain agentic benchmarks where multi-strategy interaction is prevalent. Training on \inter-86K yields higher and more stable rewards, including reduced post-peak degradation under fixed compute budgets. However, improvements are not uniform across all backbones or downstream metrics, reinforcing that ARC should be viewed as a targeted correction for comparison bias rather than a universally dominant optimizer.

Our theoretical analysis provides a mechanism-level explanation for these effects. The results in Section~\ref{sec:arc} are stylized estimator-level characterizations, not end-to-end convergence guarantees. They show how inter-strategy variance in reward-model evaluations can degrade sample efficiency, and how restricting the comparison class can reduce this variance under strategy-dependent reward bias. This supports the empirical findings without claiming to fully characterize RL training dynamics.

Overall, the results suggest that progress in open-ended agent learning depends not only on stronger models or reward signals, but also on how learning algorithms construct fair comparison sets when multiple valid behaviors coexist.

\section*{Limitations}
\label{sec:limitations}

Our study has four main limitations.

\textbf{(1) Strategy abstraction.} ARC relies on a coarse strategy taxonomy. While useful for conditioning and analysis, real interaction behaviors are more nuanced and context-dependent~\cite{lin2023toxicchat,pan2025overrefusal}.

\textbf{(2) Domain scope.} Our strongest results are in open-ended tool-usage settings. Whether ARC generalizes to other domains remains to be tested.

\textbf{(3) Data and annotation dependence.} The method depends on normalized interaction traces and strategy labels, which may reflect annotation bias or deployment-specific patterns.

\textbf{(4) Theoretical scope.} Variance and sample-efficiency analyses clarify one optimization mechanism but do not constitute a full convergence theory for open-ended agent RL.

\bibliographystyle{plainnat}
\bibliography{references}

\newpage
\appendix

\section{Proofs and Theoretical Analysis}
\label{app:proofs}

This appendix provides proofs for the stylized variance claims in Section~\ref{sec:arc}. We distinguish between the \emph{target strategy} $s^\star$, which is assigned to a training prompt before ARC rollout generation, and the \emph{realized strategy} $s_i=f(y_i)$ exhibited by a sampled response $y_i$.

\begin{equation}
\Pr_{y\sim\pi_\theta(\cdot\mid x)}[f(y)=s].
\end{equation}

Under ARC, a target strategy $s^\star$ is assigned before rollout generation and appended to the prompt. Responses are then sampled as
\begin{equation}
y_i
\overset{\mathrm{i.i.d.}}{\sim}
\pi_\theta(\cdot\mid x,s^\star).
\end{equation}
The realized strategy remains $s_i=f(y_i)$ and may differ from $s^\star$ when strategy adherence is imperfect.

\subsection{Proof of Theorem~\ref{thm:variance_amplification} (Variance Amplification)}

\begin{proof}
Decompose $r_i = \mu_{s_i} + \epsilon_i$ where $\mu_{s_i}$ is the mean reward for realized strategy $s_i$ (capturing RM bias) and $\epsilon_i$ is i.i.d.\ zero-mean noise with variance $\sigma^2_{\text{intra}}$, independent of the strategy-dependent mean term. Since $y_i \overset{\mathrm{i.i.d.}}{\sim} \pi_\theta(\cdot \mid x)$ and $s_i = f(y_i)$, the realized strategies $s_1, \ldots, s_N$ are i.i.d.\ under the induced distribution $P_\theta(\cdot \mid x)$, so $\mu_{s_1}, \ldots, \mu_{s_N}$ are i.i.d.\ with variance $\sigma^2_{\text{inter}}$.

The advantage is $\hat{A}_i = r_i - \bar{r} = (\mu_{s_i} - \bar{\mu}) + (\epsilon_i - \bar{\epsilon})$, where $\bar{\mu} = \frac{1}{N}\sum_j \mu_{s_j}$ and $\bar{\epsilon} = \frac{1}{N}\sum_j \epsilon_j$.

Since strategy and noise are independent:
\begin{equation} \small
\mathrm{Var}[\hat{A}_i] = \mathrm{Var}[\mu_{s_i} - \bar{\mu}] + \mathrm{Var}[\epsilon_i - \bar{\epsilon}]
\end{equation}

For any i.i.d.\ sequence $Z_1, \ldots, Z_N$ with variance $\sigma^2$:
\begin{equation} \small
\mathrm{Var}[Z_i - \bar{Z}] = \left(1 - \frac{1}{N}\right)^2 \sigma^2 + \frac{N-1}{N^2}\sigma^2 = \sigma^2\left(1 - \frac{1}{N}\right)
\end{equation}

Applying this to both terms yields $\mathrm{Var}[\hat{A}_i] = (\sigma^2_{\text{inter}} + \sigma^2_{\text{intra}})(1 - 1/N)$.
\end{proof}

\subsection{Proof of Theorem~\ref{thm:variance_reduction} (Variance Reduction via Conditioning)}

\begin{proof}
ARC conditions generation on the target strategy $s^\star$, and under perfect compliance, $f(y_i) = s^\star$, so $\mu_{s_i} = \mu_{s^*}$ for all $i$. The advantage simplifies to:
\begin{equation} \small
\hat{A}_i = \epsilon_i - \bar{\epsilon}
\end{equation}
Since $\epsilon_i$ are i.i.d.\ with variance $\sigma^2_{\text{intra}}$:
\begin{equation} \small
\mathrm{Var}[\hat{A}_i | s^*] = \sigma^2_{\text{intra}}\left(1 - \frac{1}{N}\right)
\end{equation}
The between-strategy mean-shift term is absent from this centered-advantage variance decomposition.

\begin{equation} \small
\begin{aligned}
\mathrm{Var}[\hat{A}_i | s^*] ={}&
\Bigl(\sigma^2_{\text{intra}}(x,s^\star) \\
&+ \operatorname{Var}_{S\sim P\theta(\cdot\mid x,s^\star)}[\mu_S(x)]\Bigr)
\left(1-\frac{1}{N}\right).
\end{aligned}
\end{equation}
With imperfect compliance, the ARC variance additionally contains the residual between-strategy term, so the predicted variance reduction is correspondingly attenuated.
\end{proof}

\subsection{Derivation for Theorem~\ref{thm:sample_complexity} (Idealized Sample-Complexity Comparison)}

\begin{arctheorem}[Gradient Variance Bound]
\label{thm:gradient_variance_app}
Let $g_{\text{ARC}}$ and $g_{\text{std}}$ denote policy gradients under strategy-conditioned and standard sampling with equal group size $N$. For analytic tractability, assume: (i)~the relevant advantage-score mixed moments approximately factorize conditional on $x$; (ii)~the score-function second moments $F = \mathbb{E}[\|\nabla_i\|^2]$ and $F' = \mathbb{E}[\langle \nabla_i, \nabla_j \rangle]$ ($i \neq j$) are comparable up to constants across the two sampling schemes; and (iii)~ARC satisfies the perfect-compliance idealization of Theorem~\ref{thm:variance_reduction}. Then:
\begin{equation} \small
\frac{\mathbb{E}[\|g_{\text{ARC}}\|^2]}{\mathbb{E}[\|g_{\text{std}}\|^2]} \approx \frac{\sigma^2_{\text{intra}}}{\sigma^2_{\text{intra}} + \sigma^2_{\text{inter}}}
\end{equation}
\end{arctheorem}

\begin{proof}
Write $g = \frac{1}{N} \sum_{i=1}^{N} \hat{A}_i \nabla_i$ where $\nabla_i = \nabla_\theta \log \pi_\theta(y_i | x)$. Under assumption~(i):
\begin{equation} \small
\mathbb{E}[\|g\|^2] = \frac{1}{N^2}\left[N \mathbb{E}[\hat{A}_i^2]\, F + N(N-1)\,\mathbb{E}[\hat{A}_i \hat{A}_j]\, F'\right]
\end{equation}

For i.i.d.\ rewards $r_i$ with total variance $\sigma^2_r$, the centered advantages satisfy:
\begin{align}
\mathbb{E}[\hat{A}_i^2] &= \sigma^2_r\left(1 - \frac{1}{N}\right) \\
\mathbb{E}[\hat{A}_i \hat{A}_j] &= -\frac{\sigma^2_r}{N} \quad (i \neq j)
\end{align}
The latter follows from $\mathrm{Cov}[r_i - \bar{r},\, r_j - \bar{r}] = -\sigma^2_r/N$.

Substituting:
\begin{equation}
\mathbb{E}[\|g\|^2] = \frac{\sigma^2_r(N-1)}{N^2}(F - F')
\end{equation}

Since $\sigma^2_r = \sigma^2_{\text{intra}} + \sigma^2_{\text{inter}}$ under standard sampling and $\sigma^2_r = \sigma^2_{\text{intra}}$ under ARC (Theorems~\ref{thm:variance_amplification}--\ref{thm:variance_reduction}), assumption~(ii) yields the stated approximation.
\end{proof}

\begin{proof}[Derivation for Theorem~\ref{thm:sample_complexity}]
This argument should be read as a stylized variance-to-sample-efficiency translation rather than a full convergence proof. Under a sub-Gaussian gradient-estimation assumption, the number of gradient samples required to achieve $\|\hat{g} - \nabla J(\theta)\| \leq \varepsilon$ with probability $\geq 1-\delta$ scales as $O(\sigma^2_g \log(1/\delta) / \varepsilon^2)$, where $\sigma^2_g$ is the per-sample gradient variance. From Theorem~\ref{thm:gradient_variance_app}, $\sigma^2_{g,\text{ARC}} / \sigma^2_{g,\text{std}} \approx \sigma^2_{\text{intra}} / (\sigma^2_{\text{intra}} + \sigma^2_{\text{inter}})$, and the same approximation carries over to this idealized sample-efficiency comparison.
\end{proof}

\subsection{Effective Variance Analysis for Curriculum Learning}

For the curriculum learning schedule studied in Section~\ref{sec:curriculum_analysis}, the expected advantage variance across the batch at training step $t$ is a mixture of conditioned and unconditioned groups:
\begin{equation}
\begin{aligned}
\mathrm{Var}_{\text{eff}}(t) ={}& \bigl(1 - p_d(t)\bigr)\,
  \sigma^2_{\text{intra}}\!\left(1 - \tfrac{1}{N}\right) \\
& {}+ p_d(t)\,\bigl(\sigma^2_{\text{intra}} +
  \sigma^2_{\text{inter}}\bigr)\!\left(1 - \tfrac{1}{N}\right).
\end{aligned}
\end{equation}

Since the second term exceeds the first by $\sigma^2_{\text{inter}}(1 - 1/N)$, $\mathrm{Var}_{\text{eff}}(t)$ increases monotonically from $\approx \sigma^2_{\text{intra}}(1 - 1/N)$ to $\approx (\sigma^2_{\text{intra}} + \sigma^2_{\text{inter}})(1 - 1/N)$. This monotonic increase gives the curriculum learning schedule a variance-based progression: the model trains under progressively noisier optimization landscapes as it becomes more capable.

\subsection{Information-Theoretic Perspective}

Assuming strategy $s$ is determined by the response $y$ (i.e., $s = f(y)$ for some deterministic function), the chain rule of mutual information gives:
\begin{equation} \small
I(r; y | x) = I(r; y | x, s) + I(r; s | x)
\end{equation}

The term $I(r; s | x)$ captures strategy-linked reward dependence that is orthogonal to within-strategy quality assessment.

Under the idealized perfect-compliance setting, conditioning on the target strategy removes between-strategy variation from within-group comparisons, emphasizing the within-strategy reward signal $I(r;y\mid x,s)$.

\section{Extended Related Work}
\label{app:extended_related_work}

\textbf{Tool-Augmented and Interactive Agents.}
Tool-augmented language models have been studied through self-supervised tool learning~\cite{schick2023toolformer}, large-scale API integration~\cite{qin2023toolllm,patil2023gorilla}, synthetic function-calling corpora~\cite{du2024apigen,toolace2024}, and increasingly sophisticated agent architectures~\cite{du2024anytool,yuan2023craft,chen2025dingtalk,tong2026onemodel}. A parallel line of work studies interaction structure, including ReAct-style reasoning-action interleaving~\cite{yao2023react}, structured planning~\cite{wei2022chain,yao2023tree,wang2023plansolve,tong2023planning}, and reflective or search-based agents~\cite{shinn2023reflexion,zhou2023lats,tong2024mistakes}. Our interest is not only in adding tools or planning steps, but in exposing interaction as a first-class, user-visible channel while execution is still ongoing.

\textbf{Interleaved Reasoning and User Experience.}
Recent interleaving methods improve perceived responsiveness by alternating internal reasoning with partial textual output~\cite{liang2025plantain,xie2025interleaved}. These works mainly address the sequencing of thought and answer tokens. \inter\ targets a different bottleneck: long-running external actions. When a tool call itself is the latency source, making the communication channel independent from the execution channel becomes the key design move. This connects to broader human-AI interaction work on transparency, feedback, and controllability~\cite{zhang2024humanai,liao2023transparency,li2024trustagent,zhang2024controllable,arenahard}, but in an agent context where the cost of opaque waiting is especially high.

\textbf{RLHF Estimators and Multi-Behavior Collapse.}
Our RL analysis sits within the broader literature on policy-gradient and RLHF optimization~\cite{ppo2017,ahmadian2024rloo,rafailov2023dpo,ethayarajh2024kto,wang2025bpo,xu2026reward,pan2026optimal}, especially group-relative estimators such as GRPO~\cite{deepseekmath2024} and refinements for entropy preservation or bias correction~\cite{yu2025dapo,liu2025drgrpo}. These methods have proven effective for reasoning-heavy tasks~\cite{deepseekr1}, but open-ended multi-behavior settings raise a distinct collapse risk~\cite{hamilton2024detecting}: the policy can converge to whichever behavior the reward favors. Our contribution is to isolate one concrete statistical mechanism for that collapse---cross-strategy contamination of relative advantages---and study a conditioning-based remedy.

\textbf{External Guidance in RL.}
Methods such as SAGE~\cite{sage2025}, Scaf-GRPO~\cite{scafgrpo2024}, LUFFY~\cite{luffy2025}, and ExGRPO~\cite{exgrpo2025} show that auxiliary guidance can materially improve RL training by mitigating sparse rewards, structuring exploration, incorporating demonstrations, or reusing successful trajectories. ARC is adjacent in form but different in purpose. The strategy instruction is not introduced to reveal the answer, densify reward, or bias the policy toward a target trajectory. It is introduced to constrain which rollouts are compared to one another, so that relative advantage estimation is performed inside a behaviorally coherent comparison class.

\section{Reward Details}
\label{app:reward_ablation}

\subsection{Final Reward Used in Main Experiments}

The specific reward used in our full \inter\ instantiation is not ARC's core mechanism; it is a setting-specific design choice for stabilizing execution-grounded training. Following recent reward-design work that combines verifiable structural signals with denser model-based feedback~\cite{toolrl2025,hero2025}, we use a compact three-term reward. The \textbf{format reward} $R_{\mathrm{fmt}} \in \{0,1\}$ checks for at least one well-formed \texttt{<tool\_call>} or \texttt{<answer>} block with properly matched, non-nested tags. The \textbf{tool reward} $R_{\mathrm{tool}} \in \{-1,0,1\}$ evaluates exact tool-call correctness: the tool identifier, argument key set, and argument values must all match the reference for $R_{\mathrm{tool}} = 1$. The \textbf{answer reward} $R_{\mathrm{ans}} \in \{0,0.50,1\}$ is an LLM-judge score on the extracted \texttt{<answer>} content. To couple semantics to correct execution, we set $R_{\mathrm{ans}}=\mathrm{Eval}(y, y^{*})$ only when $R_{\mathrm{tool}}=1$, and $R_{\mathrm{ans}}=0$ otherwise. The final reward is simply $R = R_{\mathrm{fmt}} + R_{\mathrm{tool}} + R_{\mathrm{ans}}$. If both prediction and reference contain no answer spans, we set $R_{\mathrm{ans}}=1$ to reflect structural agreement; if evaluation fails, we return a neutral fallback score.

\subsection{Detailed Reward Ablation}
\label{app:reward_study}

A central challenge in training \inter\ agents is designing rewards that evaluate multi-channel outputs---internal reasoning, tool calls, and user-facing \texttt{<answer>} spans---without baking in strategy-dependent bias. We study two reward-side choices: the reward construction itself, and the reward model used to score semantic answer quality.

\begin{table}[t]
\centering
\caption{Reward-construction ablation. All variants include the format reward. Exact tool matching improves tool use, and the final gated design gives the strongest tool-centric operating point. Unless explicitly varied, all main RL baselines and ARC variants in the paper use the final row.}
\label{tab:reward_versions}
\resizebox{\linewidth}{!}{
\begin{tabular}{lcccccc}
\toprule
\textbf{Reward Comparisons} & \textbf{Avg.} & $\tau$ \textbf{Avg.} & $\tau^2$ \textbf{Avg.} & \textbf{Reasoning Avg.} & \textbf{IFBench} & \textbf{ArenaHard} \\
\midrule
Partial tool + answer judge & 24.90 & 30.92 & 29.16 & 33.19 & 14.97 & 16.24 \\
Exact tool + answer judge & 27.89 & 37.31 & 31.94 & 36.24 & 18.03 & 15.94 \\
Penalized tool + gated judge & 27.98 & 43.93 & 36.49 & 29.60 & 16.33 & 13.57 \\
\bottomrule
\end{tabular}
}
\end{table}

The reward constructions in Table~\ref{tab:reward_versions} differ along two axes: (i) tool-execution strictness, and (ii) coupling between execution and semantic reward. All variants include a shared \textbf{format reward} enforcing structural validity. Moving from partial to exact tool matching strengthens the execution-grounded signal, and the final gated design amplifies this effect, increasing $\tau$ Avg.\ by about 42\% over the weakest variant. This comes with reduced performance on broader reasoning-style evaluations, while leaving overall average performance largely unchanged. The results reflect our objective: to prioritize rewards that favor answers grounded in correct tool execution rather than signals agnostic to whether the answer is causally supported by the executed actions.

\begin{table}[t]
\centering
\caption{Reward-model ablation for the same Qwen3-8B policy under the final reward construction used throughout the main experiments. Varying only the external judge shows that the 235B judge provides the strongest tool-use operating point.}
\label{tab:reward_model_versions}
\resizebox{\linewidth}{!}{
\begin{tabular}{lcccccc}
\toprule
\textbf{Judge Model} & \textbf{Avg.} & \textbf{$\tau$ Avg.} & \textbf{$\tau^2$ Avg.} & \textbf{Reasoning Avg.} & \textbf{IFBench} & \textbf{ArenaHard} \\
\midrule
Qwen3-235B-Instruct & 27.98 & 43.93 & 36.49 & 29.60 & 16.33 & 13.57 \\
Qwen3.5-27B & 21.67 & 25.69 & 22.55 & 27.97 & 19.39 & 12.77 \\
Qwen3.5-122B & 25.02 & 35.22 & 30.70 & 29.27 & 16.67 & 13.27 \\
\bottomrule
\end{tabular}
}
\end{table}

Table~\ref{tab:reward_model_versions} fixes the policy at Qwen3-8B and varies only the external reward judge. Under this controlled comparison, the 235B judge gives the strongest $\tau/\tau^2$ performance and the best overall balance for execution-grounded training, while smaller judges recover only isolated gains on auxiliary metrics. These findings led us to select the 235B judge for our main experiments, as it provides the most consistent signal for tool-execution verification.

\textbf{Tool matching.}
\textit{Partial tool} assigns credit to approximately correct tool calls, providing dense but noisy supervision. Replacing this with \textbf{exact tool matching} yields stricter credit assignment, improving reliability of the training signal.

\textbf{Negative penalties.}
The final variant introduces a \textbf{negative penalty} for incorrect tool execution, explicitly discouraging spurious or malformed calls and sharpening optimization toward valid trajectories.

\textbf{Semantic reward coupling.}
All variants use an LLM-based \textbf{answer judge} for user-visible responses. In the first two variants, this reward is applied unconditionally, allowing fluent but unsupported answers to receive credit. The final variant applies a \textbf{gated answer judge}, where semantic reward is issued only when exact tool execution is correct, coupling answer quality with execution validity.

\section{Entropy Analysis}
\label{app:arc_ablation}

To separate ARC's core mechanism from setting-specific stabilizers, Table~\ref{tab:combined_ablation} fixes the reward family and varies only two ingredients: strategy-conditioned grouping and entropy regularization. The first block asks which component drives the main gain over GRPO; the second block studies entropy sensitivity once grouping is enabled.

\begin{table}[t]
\centering
\caption{Ablation study on ARC: Training settings and entropy hyperparameters. Strategy-conditioned grouping provides the largest single gain over base GRPO. While entropy regularization alone slightly decreases performance, combining it with strategy-conditioned grouping further improves the overall operating point. The best entropy value (0.001) yields the highest average score.}
\label{tab:combined_ablation}
\resizebox{\linewidth}{!}{
\begin{tabular}{lcccccc}
\toprule
\textbf{Configuration} & \textbf{Avg.} & $\tau$ \textbf{Avg.} & $\tau^2$ \textbf{Avg.} & \textbf{Reasoning Avg.} & \textbf{IFBench} & \textbf{ArenaHard} \\
\midrule
\multicolumn{7}{l}{\textit{Training Ingredients}} \\
GRPO & 25.43 & 35.81 & 31.72 & 27.25 & 18.71 & 13.66 \\
GRPO + Entropy & 24.23 & 33.53 & 29.42 & 27.66 & 16.67 & 13.85 \\
GRPO + Strategy-Conditioned Grouping & 27.80 & 43.93 & 35.55 & 29.60 & 16.33 & 13.57 \\
\rowcolor{blue!10} GRPO + Strategy-Conditioned Grouping + Entropy & \textbf{29.59} & 47.00 & 38.22 & 31.58 & 15.99 & 15.18 \\
\midrule
\multicolumn{7}{l}{\textit{Entropy Sensitivity (with Strategy-Conditioned Grouping)}} \\
Entropy=0.01 (Reward Collapse) & --- & --- & --- & --- & --- & --- \\
\rowcolor{blue!10} Entropy=0.001 & \textbf{29.59} & 47.00 & 38.22 & 31.58 & 15.99 & 15.18 \\
Entropy=0.0001 & 27.48 & 40.95 & 34.11 & 32.32 & 14.29 & 15.73 \\
\bottomrule
\end{tabular}
}
\end{table}

Table~\ref{tab:combined_ablation} supports the paper's main causal story. Holding the reward design fixed, strategy-conditioned grouping is the primary source of improvement over base GRPO, whereas entropy regularization alone does not explain the gain and can even weaken performance. Entropy becomes useful only after the comparison class has been cleaned up by grouping, where it acts as a stabilizer for multi-channel generation rather than as the main mechanism. The entropy sweep further shows that this effect is sensitive to scale: too much entropy leads to reward collapse, while a moderate value of $0.001$ gives the best overall operating point. Taken together, these results suggest that ARC helps mainly by changing \emph{how} rollouts are compared, with entropy regularization serving as a secondary component that improves the stability of that mechanism in the full \inter\ recipe.

\section{Annotation Guidelines for \inter-86K}
\label{app:annotation_guidelines}

This section provides the detailed annotation guidelines used by both models and human annotators for strategy assignment.

\textbf{Progress Update}: Proactively provide progress updates to the user when simultaneously executing a tool call or multi-step tasks. Notify the user of the current progress after completing each tool call/sub-task. Let the user know the task is in progress to reduce waiting anxiety.

\textbf{Direct Answer}: Answer directly without invoking any tools if the information is already known or contextually available. Be concise and clear; avoid verbosity.

\textbf{Clarify First}: If the user's request lacks sufficient information, ask for clarification first. Do not guess the user's intent; ask directly. Execute tool calls only after the user confirms. Suitable for vague, ambiguous, or incomplete requests.

\textbf{Alignment Check}: Restate the user's requirements first to confirm your understanding is correct. Execute only after the user confirms. Avoid making mistakes due to misunderstanding. Suitable for easily misunderstood, important, or irreversible operations.



\section{Mechanism-Level Analysis and Label-Noise Robustness}
\label{app:mechanism_analysis}

To complement the idealized theoretical analysis, we perform rollout-level analyses of empirical reward unfairness, realized inter- and intra-strategy reward variance, strategy adherence, and sensitivity to strategy-label noise.

\paragraph{Realized inter- and intra-strategy reward variance.}
On actual rollouts, we estimate the realized between-strategy and within-strategy reward variances. As shown in Table~\ref{tab:variance_ratio}, the estimated inter-/intra-strategy variance ratio decreases from 0.417 under GRPO to 0.074 under ARC, corresponding to an approximately 82\% reduction. This is consistent with the proposed variance-reduction mechanism. However, the confidence intervals are wide because relatively few prompt groups contain multiple realized strategies, so we interpret this result as mechanism-level evidence rather than definitive empirical validation of the stylized theoretical analysis.

\begin{table}[H]
\centering
\caption{Estimated realized inter- and intra-strategy reward variance on rollout data. Ratio denotes $\hat{\sigma}^2_{\mathrm{inter}} / \hat{\sigma}^2_{\mathrm{intra}}$, with 95\% confidence intervals.}
\label{tab:variance_ratio}
\scalebox{0.9}{
\begin{tabular}{lccc}
\toprule
\textbf{Method}
& $\hat{\sigma}^2_{\mathrm{inter}}$
& $\hat{\sigma}^2_{\mathrm{intra}}$
& \textbf{Ratio (95\% CI)} \\
\midrule
GRPO
& 0.0555
& 0.1331
& 0.417 [0.002, 3.015] \\
ARC
& 0.0030
& 0.0401
& 0.074 [0.002, 0.410] \\
\bottomrule
\end{tabular}
}
\end{table}

\paragraph{Sensitivity to strategy-label noise.}
We additionally test the sensitivity of ARC to imperfect strategy labels by corrupting 50\% of the strategy instructions during training. For each corrupted example, the original strategy instruction is replaced with a uniformly sampled alternative strategy. This experiment uses one training seed, and the reported standard deviations are computed over three independent evaluations.

As shown in Table~\ref{tab:label_noise}, ARC remains above the GRPO baseline under substantial label corruption, with a tool-use average of 35.01 compared with 33.35 for GRPO. However, performance is substantially lower than clean ARC at 41.73, indicating that ARC benefits from accurate strategy assignments and degrades under noisy training labels.

\begin{table*}[t]
\centering
\caption{Sensitivity of ARC to strategy-label corruption. The corrupted setting replaces 50\% of strategy instructions with a uniformly sampled alternative strategy during training. Results are based on a single training run (seed 1); values after $\pm$ denote standard deviations over three independent evaluation runs.}
\label{tab:label_noise}
\resizebox{0.9\textwidth}{!}{
\begin{tabular}{lccccccc}
\toprule
\textbf{Method}
& \textbf{Tool-use Avg.}
& $\tau$-\textbf{Airline}
& $\tau$-\textbf{Retail}
& $\tau^2$-\textbf{Airline}
& $\tau^2$-\textbf{Retail}
& $\tau^2$-\textbf{Telecom} \\
\midrule
GRPO
& 33.35
& $31.33 \pm 4.62$
& $40.29 \pm 4.94$
& $36.67 \pm 6.43$
& $40.64 \pm 5.97$
& $17.84 \pm 2.82$ \\

GRPO + ARC
& 41.73
& $44.00 \pm 3.06$
& $50.00 \pm 1.33$
& $48.00 \pm 6.11$
& $45.61 \pm 0.86$
& $21.05 \pm 0.91$ \\

\textbf{GRPO + ARC, 50\% corrupted}
& \textbf{35.01}
& $\mathbf{34.00 \pm 3.46}$
& $\mathbf{42.90 \pm 1.33}$
& $\mathbf{37.16 \pm 0.73}$
& $\mathbf{42.11 \pm 3.16}$
& $\mathbf{18.87 \pm 1.07}$ \\
\bottomrule
\end{tabular}
}
\end{table*}

\section{Interaction Strategy Examples}
\label{app:strategies}

\subsection{Progress Update Strategy}

\begin{lstlisting}[
    basicstyle=\small\ttfamily,
    breaklines=true,
    breakatwhitespace=true,         % Ensures normal wrapping at spaces
    postbreak={},                   % Removes the arrow for a cleaner "normal" look
    frame=none,                     % Removes the left line if desired
    backgroundcolor=\color{white},  % Overrides the gray background
    extendedchars=true,
    literate={–}{{-}}1 {—}{{---}}3  % Maps en dash to - and em dash to ---
]
[user]
Help me analyze Acme Corp's financial status.

[assistant]
The user wants a full financial analysis. I'll start by fetching the income statement.
<tool_call>
{"name": "get_income_statement", "arguments": {"company": "Acme Corp"}}
</tool_call>
<answer>I'm retrieving Acme Corp's financial statements - step 1 of 3... </answer>

[tool]
get_income_statement -> {"revenue": "$5.2B", "net_income": "$430M"}

[assistant]
Income statement retrieved. Now fetching the balance sheet for ratio analysis.
<tool_call>
{"name": "get_balance_sheet", "arguments": {"company": "Acme Corp"}}
</tool_call>
<answer>Income statement retrieved. Now analyzing the balance sheet - step 2 of 3...</answer>
\end{lstlisting}

\subsection{Direct Answer Strategy}

\begin{lstlisting}[
    basicstyle=\small\ttfamily,
    breaklines=true,
    breakatwhitespace=true,         % Ensures normal wrapping at spaces
    postbreak={},                   % Removes the arrow for a cleaner "normal" look
    frame=none,                     % Removes the left line if desired
    backgroundcolor=\color{white},  % Overrides the gray background
]
[user]
What is the capital of China?

[assistant]
<answer>The capital of China is Beijing.</answer>
\end{lstlisting}

\subsection{Clarify First Strategy}

\begin{lstlisting}[
    basicstyle=\small\ttfamily,
    breaklines=true,
    breakatwhitespace=true,         % Ensures normal wrapping at spaces
    postbreak={},                   % Removes the arrow for a cleaner "normal" look
    frame=none,                     % Removes the left line if desired
    backgroundcolor=\color{white},  % Overrides the gray background
]
[user]
Book a meeting room for me.

[assistant]
Request lacks required information: time and capacity.
<answer>I'd be happy to book a meeting room. Could you tell me:
(1) What date and time?
(2) How many people will attend?</answer>

[user]
Tomorrow at 10 AM, 8 people.

[assistant]
Details confirmed. Proceeding with booking.
<tool_call>{"name": "book_meeting_room", "arguments": {"date": "tomorrow", "time": "10:00", "capacity": 8}}</tool_call>
<answer>Proceeding with booking a room for tomorrow at 10 AM suitable for 8 people...</answer>

[tool]
book_meeting_room -> {"status": "confirmed", "room": "Room 3B"}

[assistant]
<answer>Booked! Room 3B for tomorrow at 10:00 AM, capacity 8.</answer>
\end{lstlisting}

\subsection{Alignment Check Strategy}

\begin{lstlisting}[
    basicstyle=\small\ttfamily,
    breaklines=true,
    breakatwhitespace=true,         % Ensures normal wrapping at spaces
    postbreak={},                   % Removes the arrow for a cleaner "normal" look
    frame=none,                     % Removes the left line if desired
    backgroundcolor=\color{white},  % Overrides the gray background
    extendedchars=true,
    literate={–}{{-}}1 {—}{{---}}3  % Maps en dash to - and em dash to ---
]
[user]
Send last quarter's sales data to Manager Wang.

[assistant]
This is a send operation — I should confirm the exact scope and recipient before executing.
<answer>Just to confirm: you'd like me to send the Q1 2024 (January–March) sales data report to Manager Wang at wang@company.com. Is that correct?</answer>

[user]
Yes, that's right.

[assistant]
Intent confirmed. Executing send operation.
<tool_call>{"name": "send_email", "arguments": {"to": "wang@company.com", "subject": "Q1 2024 Sales Data", "attachment": "q1_sales.pdf"}}</tool_call>

[tool]
send_email -> {"status": "sent"}

[assistant]
<answer>Sent! Q1 2024 sales report delivered to Manager Wang.</answer>
\end{lstlisting}

\section{System Prompt}
\label{app:system_prompt}

The system prompt used for training and evaluation:

\begin{tcolorbox}[
    enhanced,
    breakable,
    boxrule=0.8pt,
    colback=white,
    colframe=gray!75!black,
    fonttitle=\bfseries,
    title=System Prompt Configuration,
]
    \subsection*{System Role}
    You are an advanced AI assistant capable of streamingly interleaved interactions, reasoning, and tool executions. Do not wait for the entire process to finish. You can provide partial updates in \texttt{<answer>} while continuing to reason and call tools.

    \vspace{1em}
    \subsection*{Output Requirements}
    Your output must be structured into functional segments. Texts that are NOT wrapped in \texttt{<answer>} tags are treated as internal reasoning and are hidden from the user. Ensure all XML-style tags (\texttt{<tool\_call>}, \texttt{<answer>}) are strictly paired and properly closed.

    \vspace{1em}
    \subsection*{Tag Definitions}
    \begin{itemize}[leftmargin=1em]
        \item \textbf{Internal Reasoning:} (No tag required) Use plain text outside of any tags for logical reasoning, task decomposition, or analyzing tool outputs. This is hidden from the user.

        \item \textbf{\texttt{<answer>}:} (User Visible) This is the ONLY content displayed to the user. Use this for status updates, partial answers, or the final conclusion. If your response relies on information from the tool outputs, do not expose the raw tool output verbatim. Instead, integrate and summarize the relevant facts naturally in your own words.

        \item \textbf{\texttt{<tool\_call>}:} (Internal Only) Use this to call external functions using the provided tools, if any.
    \end{itemize}
\end{tcolorbox}

\section{Evaluation Setup and Configuration}
\label{app:evaluation_setup}

This appendix describes the standardized evaluation configuration used across all benchmarks reported in Table~\ref{tab:main_results}.

\subsection{Inference Configuration}

All evaluations were conducted using the following controlled parameters:


\paragraph{Concurrency.}
TTFT measurements are collected with concurrency level set to 16 parallel requests on the same GPU (H200). This concurrency level is maintained consistently across all $\tau$-bench evaluations to ensure fair latency comparisons.

\paragraph{Interleaved Mode.}
When interleaved mode is enabled, we use the same structured interface as in training: \texttt{<answer>} for user-visible content and \texttt{<tool\_call>} for tool calls. This enables fine-grained measurement of time-to-first-token (TTFT) as the model can provide partial answers while continuing reasoning and tool execution.

\subsection{Benchmark Configurations}

\paragraph{$\tau$-Bench.}
Following the official $\tau$-Bench evaluation protocol~\cite{tau2024}, we evaluate tool-calling ability across the \textit{retail} and \textit{airline} domains. The user simulator is powered by the Qwen3-235B-A22B-Instruct model~\cite{qwen3}. We report the average reward score (success rate) and TTFT. Each evaluation run is repeated 3 times and we report the mean score.

\paragraph{$\tau^2$-Bench.}
Following the official $\tau^2$-Bench evaluation protocol~\cite{tau22025}, we evaluate on the \textit{retail}, \textit{airline}, and \textit{telecom} domains. The user simulator is powered by the Qwen3-235B-A22B-Instruct model~\cite{qwen3}. We repeat the evaluation 3 times and we report the mean score.

\paragraph{AIME 2026.}
The American Invitational Mathematics Examination 2026 dataset~\cite{matharena2025} contains challenging math problems with integer answers in [0, 999]. We extract the model's answer from \texttt{\textbackslash boxed\{\}} notation. Each problem is evaluated 8 times with different random seeds, and accuracy is computed by exact match after normalization. The instruction prompt asks the model to solve the problem and format the final answer.

\paragraph{GPQA Diamond.}
The Google-Proof Q\&A (GPQA) Diamond subset~\cite{gpqa2024} contains expert-level science questions with multiple-choice answers. We evaluate the model's ability to select the correct option (A, B, C, or D). Each question is run 8 times, and we compute the accuracy as the fraction of correct selections.

\paragraph{HMMT 2025.}
The Harvard-MIT Mathematics Tournament February 2025 dataset~\cite{matharena2025} contains competition-level math problems across multiple rounds. Answers may be integers, fractions, expressions, or other mathematical objects. Each problem is evaluated 8 times.

\paragraph{IFBench.}
Instruction Following Benchmark~\cite{ifbench2025} tests the model's ability to follow complex, multi-constraint instructions. The benchmark contains test cases with specific formatting, length, and content requirements. Evaluation is performed using the official IFBench scoring script, which checks constraint satisfaction.

\paragraph{Arena-Hard.}
Following the official Arena-Hard evaluation protocol~\cite{arenahard}, we use the benchmark's GPT-4-based judge to compare model responses against a baseline (o3-mini-2025-01-31) on 580 challenging questions from diverse domains. We report the win rate percentage as the final score.

\subsection{TTFT Measurement Protocol}

Time-to-first-token (TTFT) is measured specifically for tool-calling benchmarks ($\tau$-bench variants) where latency is most critical. The measurement protocol:

\begin{itemize}
    \item \textbf{Definition:} TTFT is the time from sending the request to receiving the first \textbf{user-visible} token of the model's response.

    \item \textbf{Aggregation:} For each row in Table~\ref{tab:main_results}, we report the mean TTFT over all available $\tau$-bench domain measurements for that configuration.

    \item \textbf{Conditions:} All TTFT measurements use concurrency=16 on H200 GPUs, with no other concurrent workloads.

    \item \textbf{Scope:} TTFT is not reported for reasoning benchmarks (AIME, GPQA, HMMT) or instruction-following/alignment benchmarks (IFBench, Arena-Hard) as these are primarily accuracy-focused rather than latency-sensitive.
\end{itemize}

\subsection{Evaluation Uncertainty and Training-Seed Robustness}
\label{app:uncertainty}

We additionally evaluate the statistical variability of ARC under both repeated evaluation and independent training seeds. These analyses separate variability arising from stochastic evaluation from variability arising from optimization randomness during training.

\begin{table*}[t]
\centering
\caption{Evaluation uncertainty on the tool-use benchmarks. Results are reported as mean $\pm$ standard deviation over $N=3$ independent evaluation runs using the same trained checkpoint. Tool-use Avg. is the mean over the five displayed $\tau$-bench and $\tau^2$-bench metrics.}
\label{tab:eval_uncertainty}
\scalebox{0.8}{
\begin{tabular}{lcccccc}
\toprule
\textbf{Method}
& \textbf{Tool-use Avg.}
& \textbf{$\tau$-Airline}
& \textbf{$\tau$-Retail}
& \textbf{$\tau^2$-Airline}
& \textbf{$\tau^2$-Retail}
& \textbf{$\tau^2$-Telecom} \\
\midrule
Qwen3-8B-NoThink
& $22.16 \pm 2.37$
& $12.00 \pm 0.00$
& $29.86 \pm 5.31$
& $14.61 \pm 7.79$
& $36.55 \pm 1.01$
& $17.80 \pm 2.53$ \\

Qwen3-8B-Think
& $31.35 \pm 0.19$
& $28.00 \pm 6.00$
& $36.81 \pm 4.29$
& $29.75 \pm 3.27$
& $38.71 \pm 1.06$
& $23.46 \pm 4.19$ \\
\midrule
PPO
& $33.60 \pm 1.34$
& $35.33 \pm 3.06$
& $41.45 \pm 3.92$
& $33.33 \pm 3.06$
& $38.89 \pm 2.21$
& $19.01 \pm 1.34$ \\

\textbf{PPO + ARC}
& $\mathbf{38.45 \pm 0.40}$
& $\mathbf{39.33 \pm 2.31}$
& $\mathbf{46.09 \pm 1.74}$
& $\mathbf{41.61 \pm 0.68}$
& $\mathbf{44.44 \pm 3.65}$
& $\mathbf{20.76 \pm 2.21}$ \\

DAPO
& $34.09 \pm 3.26$
& $31.33 \pm 4.16$
& $42.32 \pm 5.52$
& $35.33 \pm 7.57$
& $42.11 \pm 4.88$
& $19.37 \pm 3.40$ \\

\textbf{DAPO + ARC}
& $\mathbf{34.24 \pm 1.89}$
& $\mathbf{34.67 \pm 2.31}$
& $\mathbf{41.74 \pm 3.98}$
& $\mathbf{34.50 \pm 3.97}$
& $\mathbf{40.64 \pm 7.04}$
& $\mathbf{19.64 \pm 2.79}$ \\

GRPO
& $33.35 \pm 3.04$
& $31.33 \pm 4.62$
& $40.29 \pm 4.94$
& $36.67 \pm 6.43$
& $40.64 \pm 5.97$
& $17.84 \pm 2.82$ \\

\textbf{GRPO + ARC}
& $\mathbf{41.73 \pm 2.11}$
& $\mathbf{44.00 \pm 3.06}$
& $\mathbf{50.00 \pm 1.33}$
& $\mathbf{48.00 \pm 6.11}$
& $\mathbf{45.61 \pm 0.86}$
& $\mathbf{21.05 \pm 0.91}$ \\
\bottomrule
\end{tabular}
}
\end{table*}

\begin{table*}[t]
\centering
\caption{Robustness to training randomness. Results are reported as mean $\pm$ standard deviation over three independently trained models with different training seeds. Each seed-level score is obtained by averaging repeated evaluation runs for that checkpoint.}
\label{tab:seed_robustness}
\resizebox{0.8\textwidth}{!}{
\begin{tabular}{lcccccc}
\toprule
\textbf{Method}
& \textbf{Tool-use Avg.}
& \textbf{$\tau$-Airline}
& \textbf{$\tau$-Retail}
& \textbf{$\tau^2$-Airline}
& \textbf{$\tau^2$-Retail}
& \textbf{$\tau^2$-Telecom} \\
\midrule

PPO
& $31.81 \pm 1.87$
& $33.78 \pm 3.29$
& $36.13 \pm 5.09$
& $32.89 \pm 3.36$
& $37.14 \pm 1.92$
& $19.10 \pm 0.17$ \\

\textbf{PPO + ARC}
& $\mathbf{36.76 \pm 1.56}$
& $\mathbf{37.78 \pm 1.39}$
& $\mathbf{44.83 \pm 2.96}$
& $\mathbf{37.28 \pm 3.78}$
& $\mathbf{43.84 \pm 0.59}$
& $\mathbf{20.08 \pm 1.18}$ \\

DAPO
& $34.69 \pm 0.52$
& $35.55 \pm 4.02$
& $40.97 \pm 1.31$
& $35.56 \pm 1.02$
& $41.91 \pm 0.34$
& $19.44 \pm 0.46$ \\

\textbf{DAPO + ARC}
& $\mathbf{34.77 \pm 0.70}$
& $\mathbf{32.89 \pm 1.68}$
& $\mathbf{42.22 \pm 3.36}$
& $\mathbf{37.50 \pm 2.62}$
& $\mathbf{41.13 \pm 0.85}$
& $\mathbf{20.11 \pm 0.58}$ \\

GRPO
& $34.74 \pm 1.35$
& $32.67 \pm 1.34$
& $42.61 \pm 2.03$
& $38.20 \pm 1.37$
& $41.80 \pm 3.36$
& $18.44 \pm 0.80$ \\

\textbf{GRPO + ARC}
& $\mathbf{37.53 \pm 1.53}$
& $\mathbf{39.11 \pm 1.68}$
& $\mathbf{43.87 \pm 4.36}$
& $\mathbf{40.19 \pm 1.65}$
& $\mathbf{43.84 \pm 1.60}$
& $\mathbf{20.67 \pm 1.53}$ \\

\bottomrule
\end{tabular}
}
\end{table*}

\paragraph{Evaluation uncertainty.}
Table~\ref{tab:eval_uncertainty} reports variability across repeated evaluations of the same trained checkpoint. For the main $\tau$-bench and $\tau^2$-bench results, we perform three independent evaluation runs for each model checkpoint and report the mean and standard deviation across runs. In particular, ARC increases the tool-use average from 33.60 to 38.45 for PPO and from 33.35 to 41.73 for GRPO. The corresponding DAPO results are substantially closer, indicating that the effect of ARC depends on the underlying RL backbone.

\paragraph{Robustness to training randomness.}
Table~\ref{tab:seed_robustness} evaluates sensitivity to optimization randomness across three independent training seeds. Each resulting checkpoint is evaluated using the same repeated-evaluation protocol. We first average the repeated evaluations within each seed and then report the mean and standard deviation across the three seed-level scores.

The improvements are consistent across training seeds for PPO and GRPO. PPO improves from 31.81 to 36.76 on the tool-use average, while GRPO improves from 34.74 to 37.53. In contrast, DAPO remains essentially unchanged (34.69 versus 34.77), and we therefore do not interpret the DAPO result as evidence of a meaningful tool-use improvement.

\section{Training Details}
\label{app:training}

\subsection{Hyperparameters}

Table~\ref{tab:hyperparams} details RL training hyperparameters.

\begin{table}[t]
\centering
\caption{RL hyperparameters for the main experiments.}
\label{tab:hyperparams}
\scalebox{0.9}{
\begin{tabular}{ll}
\toprule
\textbf{Parameter} & \textbf{Value} \\
\midrule
Base Model & Qwen3-8B \\
Learning Rate & 1e-6 \\
Batch Size & 64 \\
PPO Mini Batch Size & 8 \\
PPO Micro Batch Size (per GPU) & 4 \\
Max Prompt Length & 6000 \\
Max Response Length & 1024 \\
Training Epochs & 1 \\
KL Loss Coefficient & 0.001 \\
KL Loss Type & low\_var\_kl \\
Use KL in Reward & False \\
GRPO Groups (n) & 8 \\
Entropy Coefficient & 0.001 \\
\bottomrule
\end{tabular}
}
\end{table}

\subsection{Compute Resources}
All experiments were conducted on NVIDIA H200 GPUs, using NVIDIA driver version 570.158.01 and CUDA 13.0.

\end{document}